\documentclass[11pt]{article}

\usepackage{fullpage,times,url,bm}

\usepackage{amsthm,amsfonts,amsmath,amssymb,epsfig,color,float,graphicx,verbatim}
\usepackage{algorithm,algorithmic}
\usepackage{bbm}
\usepackage{natbib}

\usepackage{graphicx}
\usepackage{subcaption}

\usepackage{array}

\usepackage{hyperref}
\hypersetup{
	colorlinks   = true, %Colours links instead of ugly boxes
	urlcolor     = blue, %Colour for external hyperlinks
	linkcolor    = blue, %Colour of internal links
	citecolor   = black %Colour of citations
}

\newtheorem{theorem}{Theorem}[section]

\newtheorem{lemma}{Lemma}[section]
\newtheorem{corollary}{Corollary}[section]

\newtheorem{remark}{Remark}[section]

\newcommand{\secref}[1]{Section~\ref{#1}}

\newcommand{\figref}[1]{Figure~\ref{#1}}
\renewcommand{\eqref}[1]{Eq.~(\ref{#1})}
\newcommand{\lemref}[1]{Lemma~\ref{#1}}

\newcommand{\corollaryref}[1]{Corollary~\ref{#1}}
\newcommand{\thmref}[1]{Theorem~\ref{#1}}

\newcommand{\appref}[1]{Appendix~\ref{#1}}

\newtheorem{open problem}[theorem]{Open Problem}

\usepackage{amssymb}

\usepackage{dsfont}

\newcommand{\stam}[1]{}

\usepackage{mathtools}

\newcommand{\bx}{\mathbf{x}}
\newcommand{\bw}{\mathbf{w}}

\newcommand{\cs}{{\cal S}}

\DeclareMathOperator*{\sign}{sign}

\newcommand{\Ncal}{\mathcal{N}}

\newcommand{\reals}{{\mathbb R}}

\newcommand{\zero}{{\mathbf{0}}}

\DeclareMathOperator{\polylog}{polylog}

\newcommand{\inner}[1]{\langle #1 \rangle}
\newcommand{\norm}[1]{\left\|#1\right\|}
\makeatletter
\newcommand{\printfnsymbol}[1]{%
  \textsuperscript{\@fnsymbol{#1}}%
}
\makeatother

\title{Adversarial Examples Exist in Two-Layer ReLU Networks for \\ Low Dimensional Linear Subspaces}

\author{
    Odelia Melamed \thanks{Weizmann Institute of Science, Israel, \texttt{odelia.melamed@weizmann.ac.il}}
    \and 
    Gilad Yehudai\thanks{Weizmann Institute of Science, Israel, \texttt{gilad.yehudai@weizmann.ac.il}}
    \and 
	Gal Vardi \thanks{TTI-Chicago and the Hebrew University of Jerusalem, \texttt{galvardi@ttic.edu}}
}

\date{}

\begin{document}

% \author[1]{Gal Vardi\thanks{equal contribution}}
% \author[2]{Gilad Yehudai\printfnsymbol{1}}
% \author[3]{Ohad Shamir}
% \affil[1]{TTI Chicago and Hebrew University, galvardi@ttic.edu}
% \affil[2]{Weizmann Institute of Science, gilad.yehudai@weizmann.ac.il}
% \affil[3]{Weizmann Institute of Science, had.shamir@weizmann.ac.il}

\maketitle

\begin{abstract}%
Despite a great deal of research, it is still not well-understood why trained neural networks are highly vulnerable to adversarial examples.
In this work we focus on two-layer neural networks trained using data which lie on a low dimensional linear subspace.
We show that standard gradient methods lead to non-robust neural networks, namely, networks which have large gradients in directions orthogonal to the data subspace, and are susceptible to small adversarial $L_2$-perturbations in these directions. 
Moreover, we show that decreasing the initialization scale of the training algorithm, or adding $L_2$ regularization, can make the trained network more robust to adversarial perturbations orthogonal to the data.
\end{abstract}

\section{Introduction}

Neural networks are observed to be susceptible to adversarial perturbations \citep{szegedy2013intriguing}, often imperceptible by humans. 
%In a seminal paper, \cite{szegedy2013intriguing} observed that neural networks are highly susceptible to adversarial perturbations, often imperceptible by humans.  
Many works have shown attacks, where adding a very small perturbation to the input may change the prediction of the network \citep{carlini2017adversarial, papernot2017practical, athalye2018obfuscated}. Other works have shown defense mechanisms
%, often 
called 
\emph{adversarial training} \citep{papernot2016distillation, madry2017towards, wong2018provable}. Despite a great deal of 
%empirical works, 
research,
it is still not well-understood why neural-network training methods tend towards such \emph{non-robust} solutions.

Several recent works have given theoretical explanation for the existence of adversarial perturbations under different settings. One line of work \citep{daniely2020most, bubeck2021single, bartlett2021adversarial,montanari2022adversarial} have shown that random networks are susceptible to adversarial perturbations. These results might explain why neural networks are non-robust at initialization, but they do not explain why trained neural networks are non-robust.
% \cite{wang2022adversarial} and \cite{dohmatob2022non} extended these results to the NTK, or \emph{lazy training} regime (see \cite{ghorbani2019limitations}).
Recently, \cite{vardi2022gradient} showed that for data which is \emph{nearly orthogonal}, after training for infinitely many iterations, the implicit bias of neural networks towards margin maximization leads to non-robust solutions. Despite these works, it is still unclear why trained neural networks tend to be non-robust to adversarial perturbations, and specifically what are the assumptions on the input data which leads to non-robustness.

One 
%very
common belief about ``real-life'' datasets, is that they approximately lie on a low dimensional ``data-manifold'' in a high dimensional space. In this setting, 
%an adversarial perturbation can be described as a perturbation in a direction orthogonal to the data-manifold. The existence of such perturbations is an undesired phenomenon, even if the perturbations are large, since it is an unwanted phenomenon to be able to change the prediction by moving orthogonal to all the possible data points in the dataset.
the existence of perturbations orthogonal to the data-manifold that change the network's predictions is especially undesired, since such perturbations do not make the input closer to data points from other classes. Indeed, such a perturbation only increases the distance between the input and all ``real-life'' examples.
\cite{shamir2021dimpled} have demonstrated empirically that under such a data-manifold assumption, the decision boundary of a trained classifier clings to the data-manifold in a way that even very small perturbations orthogonal to the manifold can change the prediction. 

In this paper we focus on data which lies on a low dimensional ``data-manifold''. 
%and that specifically this manifold is 
Specifically, we assume that the data lies on 
a linear subspace $P\subseteq \reals^d$ of dimension $d-\ell$ for some $\ell>0$.
% inside a high dimensional space, trained 2-layer ReLU neural networks are susceptible to adversarial perturbations. 
% We specifically assume that the data $x\in \reals^d$ lies on a linear subspace $P\subseteq \reals^d$ of dimension $d-\ell$. 
We study adversarial perturbations in the direction of $P^\perp$, i.e. orthogonal to the data subspace. We show that the gradient projected on $P^\perp$ is large, and in addition there exist a \emph{universal adversarial perturbation} 
%(up to norm) 
in a direction orthogonal to $P$. Namely, the same small adversarial perturbation applies to many inputs.
The norm of the gradient depends on the term $\frac{\ell}{d}$, while the perturbation size depends on the term $\frac{d}{\ell}$, i.e.
% The norm of the gradient, as well as the perturbation size, depends on the term $\frac{\ell}{d}$, i.e. 
%lower dimensional data-manifold means that the network is less robust.
a low dimensional subspace implies reduced adversarial robustness.
Finally, we also study how changing the initialization scale or adding $L_2$ regularization affects robustness. We show that in our setting, decreasing the initialization scale, or adding a sufficiently large regularization term, can make the network significantly more robust. 
%specifically, if we change the initialization variance of the weights from $\frac{1}{d}$ to $\frac{1}{d^2}$, then the network is robust to adversarial perturbations. 
We also demonstrate empirically the effects of the initialization scale and regularization on the decision boundary. Our experiments suggest that these effects might extend to deeper networks.

% Our main contributions are:
% \begin{enumerate}
%     \item For any point $x_0\in P$, the gradient projected in the direction of $P^\perp$ is of size $\Omega\left(\sqrt{\frac{\ell}{d}}\right)$ (see \thmref{largegradient}).
% \end{enumerate}

% General adversarial examples

% Recent theoretical works - either random networks or strict assumptions on the data

% The "manifold" assumption

% our contributions
\section{Related Works}

Despite extensive research, the reasons for the abundance of adversarial examples in trained neural networks are still not well understood \citep{goodfellow2014explaining,fawzi2018adversarial,shafahi2018adversarial,schmidt2018adversarially,khoury2018geometry,bubeck2019adversarial,allen2020feature,wang2020high,shah2020pitfalls,shamir2021dimpled,ge2021shift,wang2022adversarial,dohmatob2022non}. Below we discuss several prior works on this question.

In a string of works, it was shown that 
small adversarial perturbations can be found for any fixed input in certain ReLU networks with random weights (drawn from the Gaussian distribution).
Building on \cite{shamir2019simple}, it was shown in \cite{daniely2020most} that small adversarial $L_2$-perturbations can be found in random ReLU networks where each layer has vanishing width relative to the previous layer. \cite{bubeck2021single} extended this result to two-layer neural networks without the vanishing width assumption, and \cite{bartlett2021adversarial} extended it to a large family of ReLU networks of constant depth. Finally, \cite{montanari2022adversarial} provided a similar result, but with weaker assumptions on the network width and activation functions. 
These works aim to explain the abundance of adversarial examples in neural networks, since they imply that adversarial examples are common in random networks, and in particular in random initializations of gradient-based methods. However, trained networks are clearly not random, and properties that hold in random networks may not hold in trained networks.
%In this work, we consider trained networks, but our training data is in a low-dimensional subspace and hence the weights of the first layer do not change during training in directions orthogonal this subspace. Therefore, 
Our results also involve an analysis of the random initialization, but we consider the projection of the weights onto the linear subspace orthogonal to the data, and study its implications on the perturbation size required for flipping the output's sign in trained networks.

%\cite{vardi2022gradient}. Similarly to our result, they used the KKT conditions of the maximum-margin problem in parameter space, in order to prove that gradient flow converges to non-robust two-layer ReLU networks under certain assumptions. More precisely, they 

% \note{add comparison}

In \cite{bubeck2021law} and \cite{bubeck2021universal}, the authors proved under certain assumptions, that overparameterization is necessary if one wants to interpolate training data using a neural network with a small Lipschitz constant. 
Namely, neural networks with a small number of parameters are not expressive enough to interpolate the training data while having a small Lipschitz constant.
These results suggest that overparameterization might be necessary for robustness. 

\cite{vardi2022gradient}
considered a setting where the training dataset $\cs$ consists of nearly-orthogonal points, and proved that every network to which gradient flow might converge is non-robust w.r.t. $\cs$. Namely, 
building on known properties of the implicit bias of gradient flow when training two-layer ReLU networks w.r.t. the logistic loss, they proved that
%for every two-layer network that satisfies the KKT conditions of the maximum-margin problem, 
for every two-layer ReLU network to which gradient flow might converge as the time $t$ tends to infinity,
and every point $\bx_i$ from $\cs$, it is possible to flip the output's sign with a small perturbation. We note that in \cite{vardi2022gradient} there is a strict limit on the number of training samples and their correlations, as well as the training duration. Here, we have no assumptions on the number of data points and their structure, besides lying on a low-dimensional subspace. Also, in \cite{vardi2022gradient} the adversarial perturbations are shown to exist only for samples in the training set, while here we show existence of adversarial perturbation for any sample which lies on the low-dimensional manifold.

It is widely common to assume that ``real-life data'' (such as images, videos, text, etc.) lie roughly within some underlying low-dimensional data manifold. This common belief started many successful research fields such as GAN \citep{Goodfellow2014generative}, VAE \citep{kingma2013auto}, and diffusion \citep{sohl2015deep}. In \cite{fawzi2018adversarial} the authors consider a setting where the high dimensional input data is generated from a low-dimensional latent space. They theoretically analyze the existence of adversarial perturbations on the manifold generated from the latent space, although they do not bound the norm of these perturbations.
% A previous theoretical adversarial examples research concerning low-dimensional data in \cite{fawzi2018adversarial}, describes data generation from a low dimensional latent space to the high dimensional input space. They show that there exist adversarial examples on the data manifold,
% %and unconstrained to the data manifold, 
% %resulting only from 
% which depends on
% the dimension gap between the latent and input spaces. 
% %These results are completely independent of the choice of classifier, yet 
% They
% do not provide a norm bound for the adversarial perturbation. 
Previous works analyzed adversarial perturbations orthogonal to the data manifold. For example, \cite{khoury2018geometry} considering several geometrical properties of adversarial perturbation and adversarial training for low dimensional data manifolds.
% presented a geometrical framework for analyzing adversarial perturbations for low dimensional data manifolds
% presented the geometrical framework,
\cite{tanay2016boundary} analyzed theoretically such perturbations for linear networks, and \cite{stutz2019disentangling} gave an empirical analysis for non-linear models.
% In \cite{tanay2016boundary} one can see a theoretical analysis of adversarial examples  lies off the data manifold in linear models, and \cite{khoury2018geometry}, and \cite{stutz2019disentangling} referred to off-manifold adversarial examples as well. 
%There were even 
Moreover, 
%a line of 
several
experimental defence methods against adversarial examples were obtained, using projection of it onto the data manifold to eliminate the component orthogonal to the data (see, e.g., \cite{jalal2017robust,meng2017magnet,samangouei2018defense}).

\cite{shamir2021dimpled} showed empirically on both synthetic and realistic datasets that the decision boundary of classifiers clings onto the data manifold, causing very close off-manifold adversarial examples.
Our paper continues this 
%line of work, 
direction,
and provides theoretical guarantees for off-manifold perturbations on trained two-layer ReLU networks, in the special case where the manifold is a linear subspace.
% the theoretical reason for this phenomena in two-layer ReLU models using off-manifold gradients and perturbations, showing the same ideas applies in deeper ReLU networks as well.

% \cite{shamir2021dimpled} presented an idea describing the decision boundary of the classifier clinging onto the data manifold, causing the very close off-manifold adversarial examples. This work included experiments on both synthetic and high dimensional images datasets. In our paper we show the theoretical reason for this phenomena in two-layer ReLU models using off-manifold gradients and perturbations, showing the same ideas applies in deeper ReLU networks as well.

\section{Setting}
\label{frameworksection}

% \subsection{Notations}
%We denote by $x \sim \mathcal{U}(v_1,v_2,...,v_r)$ a random variable $x$ sampled from the set $\{v_1,v_2,...,v_r\}$.
%We denote by $\mathcal{U}(v_1,v_2,...,v_r)$ the uniform distribution over the set $\{v_1,v_2,...,v_r\}$.
\paragraph{Notations.} We denote by $\mathcal{U}(A)$ the uniform distribution over a set $A$. The multivariate normal distribution with mean $\mu$ and covariance $\Sigma$ is denoted by $\Ncal(\mu,\Sigma)$, and the univariate normal distribution with mean $a$ and variance $\sigma^2$ is denoted by $\Ncal(a,\sigma^2)$.
The set of integers $\{1,..,m\}$ is denoted by $[m]$.
For a vector $v \in \reals^n$, we define $v_{i:i+j}$ to be the $j+1$ coordinates of $v$ starting from $i$ and ending with $i+j$.
For a vector $x$ and a linear subspace $P$ we denote by $P^\perp$ the subspace orthogonal to $P$, and by $\Pi_{P^\perp}\left(x\right)$ the projection of $x$ on $P^\perp$. We denote by $\bm{0}$ the zero vector. We use $I_d$ for the identity matrix of size $d$.

\subsection{Architecture and Training}
In this paper we consider a two-layer fully-connected neural network $N:\mathbb{R}^d \to \mathbb{R}$ with ReLU activation, input dimension $d$ and hidden dimension $m$:
\[
N(x,\bw_{1:m})  = \sum\limits_{i=1}^{m} u_i \sigma(w_i^\top x)~.
\]
Here, $\sigma(z) = \max(z,0)$ is the ReLU function and $\bw_{1:m} = (w_1,\dots,w_m)$. When $\bw_{1:m}$ is clear from the context, we will write for short $N(x)$.
% For simplicity of the analysis we consider a network without bias terms. 

We initialize the first layer using standard Kaiming initialization \citep{he2015delving}, i.e. $w_i\sim \Ncal\left(\zero, \frac{1}{{d}}I_d\right)$, and the output layer as $u_i \sim \mathcal{U}\left(\left\{-\frac{1}{\sqrt{m}}, \frac{1}{\sqrt{m}}\right\}\right)$ for every $i\in[m]$. Note that in standard initialization, each $u_i$ would be initialized normally with a standard deviation of $\frac{1}{\sqrt{m}}$, for ease of analysis we fix the initial value to be equal to the standard deviation where only the sign is random.

We consider a dataset with binary labels.
Given a training dataset $(x_1,y_1),\dots,(x_r,y_r) \in \reals^d \times \{-1,1\}$ we train w.r.t. the logistic loss (a.k.a. binary cross entropy): $L(q) = \log(1+e^{-q})$, and minimize the empirical error:
\[
\min_{w_1,\dots,w_m} \sum_{i=1}^r L\left(y_i\cdot N(x_i,\bw_{1:m})\right)~.
\]
We assume throughout the paper that the network is trained using either gradient descent (GD) or stochastic gradient descent (SGD). 
%Our results do not depend on the training method, only that during training the weights are updated in the direction of the gradient. 
Our results hold for both training methods.
We assume that only the weights of the first layer (i.e. the $w_i$'s) are trained, while the weights of the second layer (i.e. the $u_i$'s) are fixed. 
% We note that this assumption was made in several other works, as it allows to focus on the training of the non-linear part of the network, while keeping the analysis of the dynamics tractable.

\subsection{Assumptions on the Data}
Our main assumption in this paper is that the input data lie on a low dimensional manifold, which is embedded in a high dimensional space. Specifically, we assume that this ``data manifold'' is a linear subspace, denoted by $P$, which has dimension $d-\ell$. We denote by $\ell$ the dimension of the data ``off-manifold'', i.e. the linear subspace orthogonal to the data subspace,
% manifold, 
which is denoted by $P^\perp$. In this work we study adversarial perturbations in $P^\perp$. Note that adding a perturbation from $P^\perp$ of any size to an input data point which changes its label is an unwanted phenomenon, because this perturbation is orthogonal to any possible data point from both possible labels. We will later show that under certain assumptions there exists an adversarial perturbation in the direction of $P^\perp$ which also has a small norm.
This reason for this assumption is so that the projection of the first layer weights on $P^\perp$ remain fixed during training. An interesting question is to consider general ``data manifolds'', which we elaborate on in \secref{sec:conclusions}.

To demonstrate that the low-dimensional data assumption arises in practical settings, in \figref{fig:cifar mnist variance} we plot the cumulative variance of the MNIST and CIFAR10 datasets, projected on a linear manifold.  These are calculated by performing PCA on the entire datasets, and summing over the square of the singular values from largest to smallest. For CIFAR10, the accumulated variance reaches $90\%$ at $98$ components, and $95\%$ at $216$ components. For MNIST, the accumulated variance reaches $90\%$ at $86$ components, and $95\%$ at $153$ components. This indicates that both datasets can be projected to a much smaller linear subspace, without losing much of the information.

% To justify this asusmption 
\begin{figure}[!ht]
\centering
\begin{subfigure}[b]{0.45\textwidth}
    \centering
     \includegraphics[width=\textwidth]{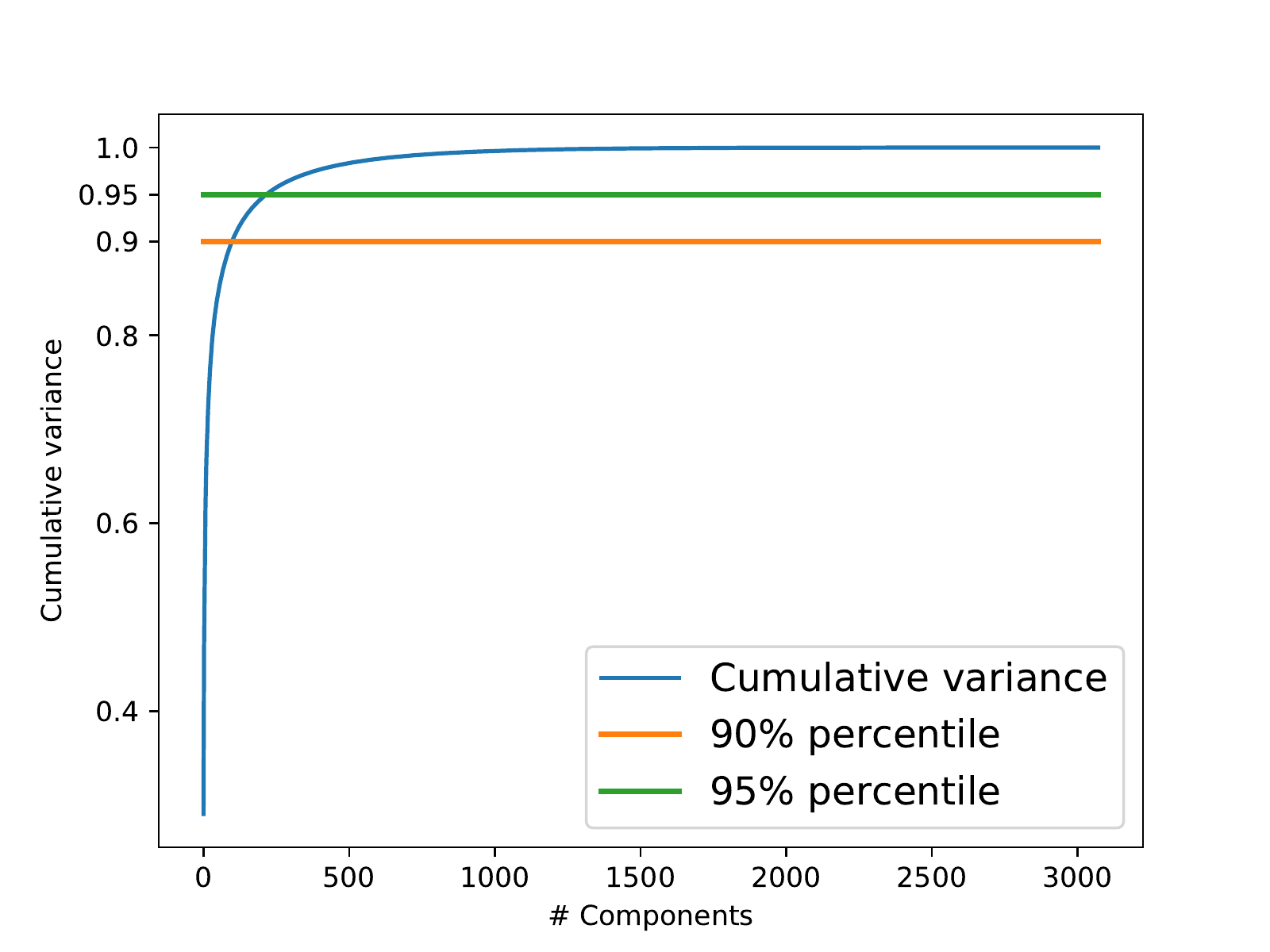}
     \caption{}
     \label{fig: cifar}
\end{subfigure}
\begin{subfigure}[b]{0.45\textwidth}
    \centering
     \includegraphics[width=\textwidth]{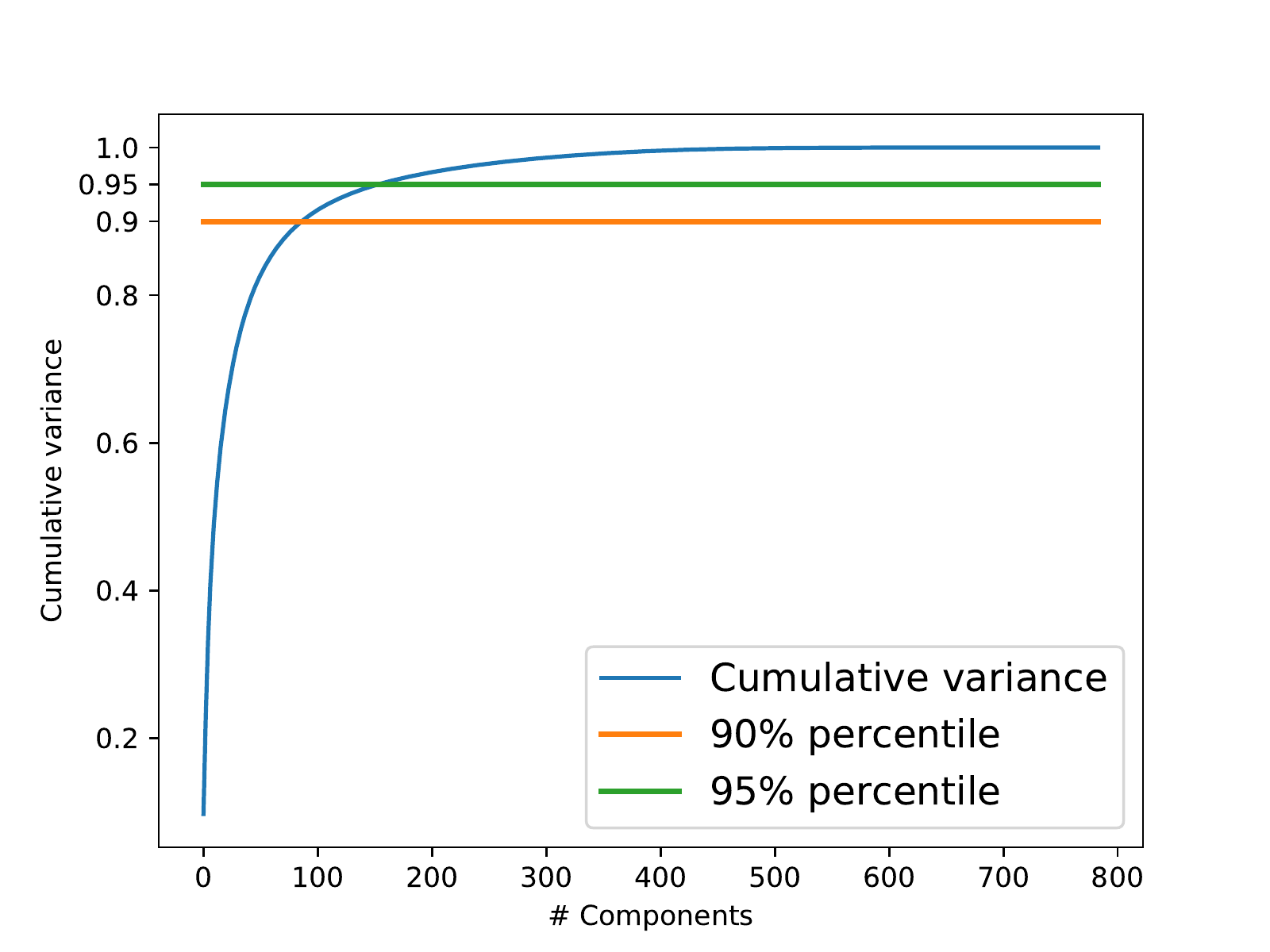}
     \caption{}
     \label{fig: mnist}
\end{subfigure}
% \subfigure[]{\includegraphics[width=0.3\textwidth]{imgs/subfigures/chess-clean.png}\label{fig: 2d clean}}
% \subfigure[] {\includegraphics[width=0.3\textwidth]{imgs/subfigures/chess-smallinit.png}\label{fig: 2d small init}}
% \subfigure[]{\includegraphics[width=0.3\textwidth]{imgs/subfigures/chess-ragular.png}\label{fig: 2d regular}}
\caption{The cumulative variance for the (a) CIFAR10; and (b) MNIST datasets, calculated by performing  PCA on the entire datasets, and summing over the square of the singular values from largest to smallest.}\label{fig:cifar mnist variance}
\end{figure}

% Our second assumption on the input data is that $\norm{x} = \Theta(\sqrt{d})$ for every data point $x$ in the dataset (both train and test). This assumption is reasonable if we assume that each coordinate of a data point is of magnitude $\Theta(1)$. We could have other assumptions on the data, such as different scaling, which would also scale our results. We note that some of our results does not depend on the scale of the data, but rather the scale of the initialization, this is discussed in more details in sections \note{add ref}.

\begin{remark}[On the Margin of the Network]\label{remark:margin}
    Given a neural network $N:\reals^d\rightarrow\reals$ and a dataset $(x_1,y_1),\dots,(x_r,y_r)$ with binary labels which the network label correctly, we define the margin of the network as $\gamma:=\min_{i\in [r]}y_iN(x_i)$. 
    % For simplicity, if the network labels incorrectly some data point (i.e $\gamma < 0$), then we just write $\gamma = 0$. 
    
    In our setting, it is possible to roughly estimate the margin without assuming much about the data, besides its boundedness. Note that the gradient of the loss decays exponentially with the output of the network, because
    % $\left|\frac{\partial \ell(y,\hat{y})}{\partial \hat{y}}\right| \leq \left|\frac{-\hat{y}e^{-y\hat{y}}}{1 + e^{-y\hat{y}}} \right|$.
    $\left|\frac{\partial L(q)}{\partial q}\right| = \left|\frac{-qe^{-q}}{1 + e^{-q}} \right|$.
    Hence, if we train for at most polynomially many iterations and label all the data points correctly (i.e. the margin is larger than $0$), then training effectively stops after the margin reaches $O(\log^2(d))$. This is because if the margin is $\log^2(d)$, then the gradient is of size:
    \[
         \left|L'(\log^2(d))\right| = \left|\frac{-\log^2(d)e^{-\log^2(d)}}{1 + e^{-\log^2(d)}}\right| \leq \log^2(d)\cdot d^{-\log(d)}~,
    \]
    % the loss is of size $e^{-\log^2(d)} = d^{-\log(d)}$
    which is smaller than any polynomials in $d$.
    % Indeed, if the margin reaches $O(\log^2(d))$, then by the above calculation of the gradient, it is upper bounded by $O\left(\frac{log(d)}{d^2}\right)$. 
    This means that all the data points on the margin (which consists of at least one point) will have an output of $O(\polylog(d))$. 
    % Note that changing the label of a data point requires changing the sign of the network's output on this data point, which seems plausible for samples on the margin since the network's output is not very large. 
    
    The number of points which lie exactly on the margin is difficult to assess, since it may depend on both the dataset and the model. Some empirical results in this direction are given in \cite{vardi2022gradient}, where it is observed (empirically) that for data sampled uniformly from a sphere and trained with a two-layer network, over $90\%$ of the input samples lie on the margin. Also, in \cite{haim2022reconstructing} it is shown that for CIFAR10, a large portion of the dataset lies on the margin.
\end{remark}

\section{Large Gradient Orthogonal to the Data Subspace}\label{sec:large gradient}

%let x1...xh data points on P, transforming into data from M
%remove the word sample
%change the rotation thm for rotating back the weights, and then say that norm and angles are rotation invariant

One proxy for showing non-robustness of models, is to show that their gradient w.r.t. the input data is large (cf. \cite{bubeck2021law,bubeck2021universal}). 
%this was done in previous works, e.g. \cite{bubeck2021law}. 
Although a large gradient does not guarantee that there is also an adversarial perturbation, it is an indication that a small change in the input might significantly change the output. Moreover, by assuming smoothness of the model, it is possible to show that having a large gradient may suffice for having an adversarial perturbation.

In this section we show that training a network on a dataset which lies entirely on a linear subspace yields a large gradient in a direction which is orthogonal to this subspace. Moreover, the size of the gradient depends on the dimension of this subspace. Specifically, the smaller the dimension of the data subspace,
% ``data manifold'',
the larger the gradient is in the orthogonal direction. Our main result in this section is the following:

% For our main theorem, we look at our two-layer neural network $N(x) = \sum\limits_{i=1}^{m} u_i \sigma(w_i^Tx) $ and its input space $ \reals^d $. We denote by $P$ the linear $d-\ell$ dimensional data-manifold and train the network on a dataset $X$ lies on $P$. We look at a point $x_0$ in $P$. We claim that the network's gradients at $x_0$, projected onto the subspace orthogonal to the manifold $P$, have large norm with high probability.

\begin{theorem}
\label{largegradient}
       Suppose that a network $N(x) = \sum\limits_{i=1}^{m} u_i \sigma(\inner{w_i,x}) $ is trained on a dataset which lies on a linear subspace $P \subset \reals^d$ of dimension $d-\ell$ for $\ell \geq 1$, and let $x_0 \in P$. Let $S= \{i\in[m]: \inner{w_i,x_0} \geq 0\}$, and let $k := |S|$.
    %   We say that a neuron $i\in[m]$ is active w.r.t $x_0$ if $w_i^\top x_0 \geq 0$. Let $S\subseteq [m]$ be the set of active neurons of $N$ w.r.t $x_0$, and let $k = |S|$. 
    %   Let $\hat{x}_0 = \Pi_{P^\perp}\left(x_0\right)$ be the projection of $x_0$ on $P^\perp$ (the subspace orthogonal to $P$), and define  $\Pi_{P^\perp} \left(\frac{\partial N(x_0)}{\partial x_0}\right)$ be the projection of the gradient of $N$ w.r.t $x_0$ on $P^\perp$.
    Then, w.p $\geq 1 - e^{-\ell/16}$ (over the initialization) we have:
    \[\norm{\Pi_{P^\perp}\left(\frac{\partial N(x_0)}{\partial x}\right)} \geq \sqrt{\frac{k \ell}{2 m d}}~. \]

 %we need oproject the geadient
\end{theorem}

The full proof can be found in \appref{appen:proofs from large gradient}. Here we provide a short proof intuition: First, we use a symmetry argument to show that it suffices to consider w.l.o.g. the subspace $M:=\text{span}\{e_1,\dots,e_{d-\ell}\}$, where $e_i$ are the standard unit vectors.
%and training on the same dataset, but rotated using an appropriate orthogonal matrix. 
Next, we note that since the dataset lies on $M$, only the first $d-\ell$ coordinates of each weight vector $w_i$ are trained, while the other $\ell$ coordinates are fixed at their initial value. Finally, using standard concentration result on Gaussian random variables we can lower bound the norm of the gradient. Note that our result shows that there might be a large gradient orthogonal to the data subspace. This correspond to ``off-manifold'' adversarial examples, while the full gradient (i.e. without projecting on $P^\perp$) might be even larger.

The lower bound on the gradient depends on two terms: $\frac{k}{m}$ and $\frac{\ell}{d}$. The first term is the fraction of active neurons for the input $x_0$, i.e. the neurons whose inputs are non-negative. Note that inactive neurons do not increase the gradient, since they do not affect the output. The second term corresponds to the fraction of directions orthogonal to the data.
%size of the "off-manifold", i.e. the dimension of the subspace which is orthogonal to the data manifold. 
The larger the dimension of the orthogonal subspace, the more directions in which it is possible to perturb the input while still being orthogonal to the data. If both of these terms are constant, i.e. there is a constant fraction of active neurons, and ``off-manifold'' directions, we can give a more concrete bound on the gradient:

\begin{corollary}\label{cor:large gradient}
For $\ell = \Omega(d)$, $k = \Omega(m)$, in the setting of \thmref{largegradient}, with probability 
%$\geq 1- e^{-\ell/16}$ we get that:
$\geq 1- e^{-\Omega(d)}$ we have:
\[ \norm{\Pi_{P^\perp}\left(\frac{\partial N(x_0)}{\partial x}\right)}  = \Omega(1)~. \]
\end{corollary}

% Note that by Remark \ref{remark:margin}, for a point on the margin (or close to it), its output is of size $\polylog(d)$. Hence, under the assumptions of \corollaryref{cor:large gradient}, the gradient is very large w.r.t. the size of the input. 
Consider the case where the norm of each data point is $\Theta(\sqrt{d}) = \Theta(\sqrt{\ell})$, i.e. every coordinate is of size $\Theta(1)$. By Remark \ref{remark:margin}, for a point $x_0$ on the margin, its output is of size $\polylog(d)$. 
%\corollaryref{cor:large gradient} shows that for points on the margin, if there exists an adversarial perturbation in the direction which the gradient is large, then the size of this perturbation would be $\polylog(d)$, which is much smaller than the size of a data point which is $\Theta(\sqrt{d})$.
Therefore, for the point $x_0$, gradient of size $\Omega(1)$ corresponds to an adversarial perturbation of size $\polylog(d)$, which is much smaller than $\norm{x_0}=\Theta(\sqrt{d})$. We note that this is a rough and informal estimation, since, as we already discussed, a large gradient at $x_0$ does not necessarily imply that an adversarial perturbation exists. 
%Nevertheless, the gradient's size is often used as a yardstick for measuring adversarial robustness (see, e.g., \cite{bubeck2021law,bubeck2021universal}).
In the next section, we will prove the existence of adversarial perturbations.

\section{Existence of an Adversarial Perturbation}\label{sec:adv pert exists}

%to add \alpha z instead of z.
In the previous section we have shown that at any point $x_0$ which lies on the linear subspace of the data $P$, there is a large gradient in the direction of $P^\perp$. In this section we show that not only the gradient is large, there also exists an adversarial perturbation in the direction of $P^\perp$ which changes the label of a data point from $P$ (under certain assumptions). 
%This perturbation will be ``off-manifold'' in the sense that it will lie entirely on $P^\perp$. 
The main theorem of this section is the following:
% This perturbation will be in the direction of the gradient projected on $P^\perp$, i.e. it will be an "off-manifold" perturbation. The main theorem of this section is the following:

% We showed that there is a large gradient at any $x_0$ that lies on the manifold $P$. Next, we show that this large gradient also leads to a small adversarial perturbation. Particularly, we show that there exists a small-norm off-manifold perturbation $z$ such that for any data point $x_0$, it changes the network classification $\text{sign}(N(x_0+z)) \neq \text{sign}(N(x_0))$.

% Let $x_0$ be a data point lies on $P$ and assume that $y_0 = N(x_0) = 1$ (the proof will be similar for $y_0=N(x_)=0$). We want to decrease $y_0$ using a perturbation $z$ in $P^\perp$, i.e., the subspace orthogonal to $P$. As in the previous section, for a weight vector $w_i \in \reals^d$ for $i \in [m]$ we denote by $\hat{w}_i:= \Pi_{P^\perp}\left(w_i\right)$ the projection of $w_i$ on $P^\perp$. 

% Next, we will direct the perturbation $z$ to correlate with the $\hat{w}_i$s for which $u_i$ is negative and anti-correlate with the $\hat{w}_i$s for which $u_i$ are positive. 

% Note, while decreasing a positive input to the Relu layer, we might cross to its zero function part. Therefore, any effect we will measure using the positive $\hat{w}_i$s might zero out. For this reason, we will analyze only the  effect we get by correlation to the direction of the negative weight, assuming that the anti-correlation with the positive weights might not help but will surely not hurt us.

\begin{theorem}\label{thm: exists pert}
Suppose that a network $N(x) = \sum\limits_{i=1}^{m} u_i \sigma(\inner{w_i, x})$ is trained on a dataset which lies on a linear subspace $P\subseteq \reals^d$ of dimension $d-\ell$, where $\ell \geq 32(m-1)\log(m^2 d)$. Let $x_0\in P$, and denote $y_0:=\text{sign}(N(x_0))$. Let $I_- := \{i \in [m] : u_i < 0\}$ and $I_+ := \{i \in [m] : u_i > 0\}$, and denote $k_- := |\{i\in I_- : \inner{w_i,x_0} \geq 0\}|$, and $k_+ := |\{i\in I_+ : \inner{w_i,x_0} \geq 0\}|$. Let $k_{y_0} = k_-$ if $y_0 = 1$ and $k_{y_0}=k_+$ if $y_0 = -1$. For $w\in \reals^d$ denote $\hat{w}:= \Pi_{P^\perp}(w)$, and denote the perturbation $z := y_0 \cdot \alpha \left( \sum\limits_{i \in I_-} \hat{w}_i - \sum\limits_{i \in I_+} \hat{w}_i \right)$ where $\alpha = \frac{8\sqrt{m}dN(x_0)}{\ell k_{y_0}}$.
% if $y_0 = 1$
% and $\alpha = \frac{8\sqrt{m}dN(x_0)}{\ell k_+}$ if $y_0 = -1$. 
Then, w.p. $\geq 1 - 5(me^{-\ell/16} + d^{-1/2})$ we have that $\norm{z} \leq 8\sqrt{2}N(x_0)\cdot \frac{m}{k_{y_0}}\cdot \sqrt{\frac{d}{\ell}}$ and:
\[
\text{sign}(N(x_0 + z)) \neq \text{sign}(N(x_0))~.
\]

% for some $\alpha > 0$. Then, w.p. $\geq 1- 4(me^{-\ell/16} + d^{-1/2})$, if $y_0 = 1$ we have:
%     \[
%     N(x_0+z) \leq -C k_- \cdot \frac{1}{\sqrt{m}}  + N(x_0)~,
%     \]
% and if $y_0 = -1$ we have:
%     \[
%     N(x_0+z) \geq C k_+ \cdot \frac{1}{\sqrt{m}}  + N(x_0)~,
%     \]
%     where $C := \alpha \left( \frac{1}{2} \frac{\ell}{d} - \sqrt{2} \sqrt{m-1} \frac{\sqrt{\ell}}{d} \sqrt{\log(dm^2 )} \right)$.
\end{theorem}

% \begin{theorem}\label{thm: exists pert}
% Let $N(x) = \sum\limits_{i=1}^{m} u_i \sigma(w_i^\top x)$ be a network trained on a dataset which lies on a linear subspace $P\subseteq \reals^d$ of dimension $d-\ell$, where $\ell \geq 32(m-1)\log(m^2 d)$. Let $x_0\in P$, and denote $y_0:=\text{sign}(N(x_0))$. Denote by $I_- := \{i \in [m] : u_i < 0\}$ and $I_+ := \{i \in [m] : u_i \geq 0\}$, also denote by $k_- := |\{i\in I_- : \inner{w_i,x_0} > 0\}|$ and $k_+ := |\{i\in I_+ : \inner{w_i,x_0} > 0\}|$. For $w\in \reals^d$ denote $\hat{w}:= \Pi_{P^\perp}(w)$, and denote the perturbation $z := y_0 \cdot \alpha \left( \sum\limits_{i \in I_-} \hat{w}_i - \sum\limits_{i \in I_+} \hat{w}_i \right)$ for some $\alpha > 0$. Then, w.p. $\geq 1- 4(me^{-\ell/16} + d^{-1/2})$, if $y_0 = 1$ we have:
%     \[
%     N(x_0+z) \leq -C k_- \cdot \frac{1}{\sqrt{m}}  + N(x_0)~,
%     \]
% and if $y_0 = -1$ we have:
%     \[
%     N(x_0+z) \geq C k_+ \cdot \frac{1}{\sqrt{m}}  + N(x_0)~,
%     \]
%     where $C := \alpha \left( \frac{1}{2} \frac{\ell}{d} - \sqrt{2} \sqrt{m-1} \frac{\sqrt{\ell}}{d} \sqrt{\log(dm^2 )} \right)$.
% \end{theorem}

The full proof can be found in \appref{appen:proofs from adv perf exists}. Here we give a short proof intuition: As in the previous section, we show using a symmetry argument that w.l.o.g. we can assume that $P=\text{span}\{e_1,\dots,e_{d-\ell}\}$.
% the data lies on the manifold $M=\text{span}\{e_1,\dots,e_{d-\ell}\}$, where $e_i$ are the standard unit vectors. We also show that under this assumption, the projection of the weights on $M^\perp$ do not change during training. Thus, instead of considering the perturbation $z$ as in the theorem, we can consider the perturbation 
% \[
% \tilde{z} := y_0\cdot \alpha\left(\sum_{i\in I_-} \Pi_{M^\perp}(w_i) - \sum_{i\in I_+} \Pi_{M^\perp}(w_i)  \right)
% \]
% which is just a rotation of $z$ to align with the rotated data.

Now, given the perturbation ${z}$ from \thmref{thm: exists pert} we want to understand how adding it to the input changes the output. Suppose that $y_0 = 1$. We can write 
\[
    N(x_0 + {z}) = \sum_{i\in I_-} u_i\sigma(\inner{w_i,x_0} + \inner{w_i,{z}}) + \sum_{i\in I_+} u_i\sigma(\inner{w_i,x_0} + \inner{w_i,{z}})
\]
We can see that for all $i$:
\begin{align*}
\inner{w_i,{z}} &=\alpha \cdot \inner{w_i, \sum_{j\in I_-} \Pi_{P^\perp}(w_j) - \sum_{j\in I_+} \Pi_{P^\perp}(w_j)} \\
& = -\alpha \cdot \inner{w_i, \sum_{j=1}^m\text{sign}(u_j)\Pi_{P^\perp}(w_j)}.
\end{align*}
For $i\in I_-$ we can write: 
\begin{align*}
\inner{w_i,{z}} &=
% \alpha \cdot \inner{w_i, \sum_{j\in I_-} \Pi_{P^\perp}(w_j) - \sum_{j\in I_+} \Pi_{P^\perp}(w_j)} \\
% & = -\alpha \cdot \inner{w_i, \sum_{j=1}^m\text{sign}(u_j)\Pi_{P^\perp}(w_j)}\\
% & =
\alpha\norm{\Pi_{P^\perp}(w_i)}^2 - \alpha\inner{\Pi_{P^\perp}(w_i), \sum_{j\neq i}\text{sign}(u_j)\Pi_{P^\perp}(w_j) }~,  
\end{align*}
% \[
% \inner{w_i,\tilde{z}} = \norm{\Pi_{M^\perp}(w_i)}^2 + \inner{\Pi_{M^\perp}(w_i), \sum_{j\neq i}\text{sign}(u_i)\Pi_{M^\perp}(w_i) }
% \]
and using a similar calculation, for $i\in I_+$ we can write: 
\begin{align*}
\inner{w_i,{z}} &=
% \alpha \cdot \inner{w_i, \sum_{j\in I_-} \Pi_{P^\perp}(w_j) - \sum_{j\in I_+} \Pi_{P^\perp}(w_j)} \\
% & = -\alpha \cdot \inner{w_i, \sum_{j=1}^m\text{sign}(u_j)\Pi_{P^\perp}(w_j)}\\
% & = 
-\alpha\norm{\Pi_{P^\perp}(w_i)}^2 - \alpha\inner{\Pi_{P^\perp}(w_i), \sum_{j\neq i}\text{sign}(u_j)\Pi_{P^\perp}(w_j) }~.  
\end{align*}
% \[
% \inner{w_i,\tilde{z}} = -\alpha\norm{\Pi_{M^\perp}(w_i)}^2 - \alpha\inner{\Pi_{M^\perp}(w_i), \sum_{j\neq i}\text{sign}(u_j)\Pi_{M^\perp}(w_j) }~.
% \]
Using concentration inequalities of Gaussian random variables, and the fact that $\Pi_{P^\perp}(w_i)$ did not change from their initial values, we can show that:
\[
\left|\inner{\Pi_{P^\perp}(w_i), \sum_{j\neq i}\text{sign}(u_j)\Pi_{P^\perp}(w_j) }\right|\approx \frac{\sqrt{\ell m}}{d}~,
\]
while $\norm{\Pi_{P^\perp}(w_i)}^2 \approx \frac{\ell}{d}$. Thus, for a large enough value of $\ell$ we have that $\inner{w_i,{z}} \leq 0$ for $i\in I_+$ and $\inner{w_i,{z}} \approx \alpha\cdot \frac{\sqrt{\ell}}{d}\cdot(\sqrt{\ell} - \sqrt{m})$
% \frac{\ell - \sqrt{\ell m }}{d}$ 
for $i\in I_-$. 

From the above calculations we can see that adding the perturbation ${z}$ does not increase the output of the neurons with a positive second layer. On the other hand, adding $z$ can only increase the input of the neurons with negative second layer, and for those neurons which are also active it increases their output as well if we assume that $\ell > m$. This means, that if there are enough active neurons with a negative second layer (denoted by $k_-$ in the theorem), then the perturbation can significantly change the output. In the proof we rely only on the active negative neurons to change the label of the output (for the case of $y_0=1$, if $y_0=-1$ we rely on the active positive neurons). Note that the active positive neurons may become inactive, and the inactive negative neurons may become active. Without further assumptions it is not clear what is the size of the perturbation to make this change for every neuron. Thus, the only neurons that are guaranteed to help change the label are the active negative ones, which by our assumptions on $\ell$ are guaranteed to increase their output.

% we would not know by how much that would need to 
% along the perturbation, while the inactive neurons, both positive and negative, may 
% We agree that t
% We note that the assumption on $\ell$ and the network width $m$ is limiting when using very wide over-parameterized networks (e.g. in the NTK  regime \cite{jacot2018neural}), but we believe it is reasonable in general. First, we note that such an assumption is also used in other theoretical papers about adversarial perturbations (e.g. \cite{daniely2020most}). In addition, we believe that for most practical datasets the orthogonal subspace dimension $\ell$ is close to $d$, making the width bound less restrictive. 

% Also, n
Note that our perturbation is \emph{not} in the direction of the gradient w.r.t. the input. The direction of the gradient would be the sum of all the active neurons, i.e. the sum (with appropriate signs) over all 
%$w_i$'s 
$i \in [m]$
such that $\inner{w_i,x_0} \geq 0$. Our analysis would not have worked with such a perturbation, because we could not guarantee that inactive neurons would stay inactive. 

The assumption that $\ell = \Omega(M)$ (up to log factors) is a technical limitation of our proof technique. We note that such an assumption is also used in other theoretical papers about adversarial perturbations (e.g. \cite{daniely2020most}). 

Note that the direction of the perturbation ${z}$ does not depend on the input data $x_0$, only its size depends on $x_0$. In fact, \thmref{thm: exists pert} shows that there is a single universal direction for an adversarial perturbation that can flip the label of any data point in $P$. The size of the perturbation depends on the dimension of the linear subspace of the data, the number of active neurons for $x_0$, the total number of neurons in the network and the size of the output. In the following corollary we give a specific bound on the size of the perturbations under assumptions on the different parameters of the problem:

\begin{corollary}\label{cor:adv pert specific size}
    In the setting of \thmref{thm: exists pert}, assume in addition that $\ell = \Theta(d)$ and $k_{y_0} = \Theta(m)$. Then, there exists a perturbation $z$ such that w.p. $\geq 1- 5\left(me^{-\ell/16} + d^{-1/2}\right)$ we have $\norm{z} = O(N(x_0))$ and:
    \[
    \text{sign}(N(x_0 + z)) \neq \text{sign}(N(x_0))~.
    \]
\end{corollary}
The above corollary follows directly by noticing from \thmref{thm: exists pert} that: 
\[
\norm{z} \leq O\left(N(x_0)\cdot \frac{m}{k_{y_0}}\cdot \sqrt{\frac{d}{\ell}}\right) = O(N(x_0))~,
\]
where we plugged in the additional assumptions.
%
% The full proof can be found in \appref{appen:proofs from adv perf exists}. The proof is a straightforward use of \thmref{thm: exists pert}, while also bounding the norm of $z$.
% by choosing an appropriate $\alpha$. 
% We will discuss the assumption on the problem's parameters (i.e. assumptions (1)-(5) in \corollaryref{cor:adv pert specific size}) in more details:
The assumptions in the corollary above are similar to the assumptions in \corollaryref{cor:large gradient}. Namely, that the dimension of the data subspace is a constant fraction from the dimension of the entire space, and the number of active neurons is a constant fraction of the total number of neurons. Note that here we only consider active neurons with a specific sign in the second layer.
Note that the size of the perturbation in \corollaryref{cor:adv pert specific size} is bounded by $N(x_0)$. By Remark~\ref{remark:margin}, the output of the network for data points on the margin can be at most $O(\log^2(d))$, since otherwise the network would have essentially stopped training. 
Therefore, if we consider an input $x_0$ on the margin, and $\norm{x_0}=\Theta(\sqrt{d})=\Theta(\sqrt{\ell})$, then the size of the adversarial perturbation is much smaller than $\norm{x_0}$.
For any other point, without assuming it is on the margin, and since we do not assume anything about the training data (except for being in $P$), we must assume that the size of the perturbation required to change the label will depend on the size of the output.

\section{The Effects of the Initialization Scale and Regularization on Robustness}\label{sec: improve robustness}
In \secref{sec:adv pert exists}, we presented the inherent vulnerability of trained models to small perturbations in a direction orthogonal to the data subspace. In this section, 
we return to a common proxy for robustness that we considered in \secref{sec:large gradient} -- the gradient at an input point $x_0$.
We suggest two ways that might improve the robustness of the model in the direction orthogonal to the data, by decreasing an upper bound of the gradient in this direction. 
% Note that a small gradient does not suffice to deduce robustness without additional assumptions on the smoothness of the function. 
% In our paper, we avoid assumptions about the data and the weights within the data subspace in order to keep it as generic as possible. Therefore, we cannot prove smoothness related properties is our settings since the weights of the model might be such that the gradient fluctuate in a small environment of an input $x_0$. 
% In this section,
% Even though, we look at a common proxy for robustness -- small gradient at an input point $x_0$.
We first upper bound the gradient of the model in the general case where we initialize $w_i \sim \mathcal{N}(\zero,\beta^2 I_d)$, and later discuss strategies to use this upper bound for improving robustness.

\begin{theorem}
\label{thm:upper bound small init}
    Suppose that a network $N(x) = \sum\limits_{i=1}^{m} u_i \sigma(\inner{w_i, x})$ is trained on a dataset which lies on a linear subspace $P\subseteq \reals^d$ of dimension $d-\ell$ for $\ell \geq 1$, and assume that the weights $w_i$ are initialized from $\mathcal{N}(\zero,\beta^2I_d)$. Let $x_0 \in P$, let $S= \{i\in[m]: \inner{w_i,x_0} \geq 0\}$, and let $k = |S|$.
    %Let the network $N(x) = \sum\limits_{i=1}^{m} u_i \sigma(\inner{w_i,x}) $ for $w_i$ initialized from $\mathcal{N}(0,\beta^2)$, trained on dataset $X$ which lies on linear subspace $P \subset \reals^d$ of dimension $d-\ell$ for $\ell \geq 1$, and let $x_0 \in P$. Let $S= \{i\in[m]: \inner{w_i,x_0} \geq 0\}$, and let $k := |S|$.
    Then, w.p. $\geq 1 - e^{-\frac{\ell}{16}}$ we have:
    \[\norm{\Pi_{P^\perp}\left(\frac{\partial N(x_0)}{\partial x}\right)} \leq \beta \sqrt{\frac{2 k \ell}{ m}}~. \]
\end{theorem}

The full proof uses the same concentration bounds ideas as the lower bound proof and can be found in \appref{apen: proofs for improve robustness}. This bound is a result of the untrained weights: since the projection of the data points on $P^\perp$ is zero, the projection of the weights vectors on $P^\perp$ are not trained and are fixed at their initialization. 
% Note, in the proof, we first rotate our data subspace $P$ to the data subspace $M:=\text{span}\{e_1,\dots,e_{d-\ell}\}$, and rotate the weights of the model accordingly. The projection of the weights on the subspace orthogonal to $M$ is essentially an $\ell$-dimensional vector since the other coordinates will be zeros. Therefore, the upper bound is independent of the input dimension.
We note that \thmref{largegradient} readily extends to the case of initialization from  $\mathcal{N}(\zero,\beta^2I_d)$, in which case the lower bound it provides matches the upper bound from \thmref{thm:upper bound small init} up to a constant factor.
% Note that for $\beta = \frac{1}{\sqrt{d}}$, this upper bound matches the lower bound from \thmref{largegradient}, up to a constant factor.
In what follows, we suggest two ways to affect the post-training weights in the $P^\perp$ direction: (1) To initialize the weights vector using a smaller-variance initialization, and (2) Add an $L_2$-norm regularization on the weights. We next analyze their effect on the upper bound. 

\subsection{Small Initialization Variance}
From \thmref{thm:upper bound small init}, one can conclude a strong result about the model's gradient without the dependency of its norm on $\ell$ and $k$.

\begin{corollary}
For $\beta = \frac{1}{d\sqrt{2}}$, in the settings of \thmref{thm:upper bound small init}, with probability
$\geq 1 - e^{-\frac{\ell}{16}}$ we have

\[\norm{\Pi_{P^\perp}\left(\frac{\partial N(x_0)}{\partial x}\right)} \leq \frac{1}{\sqrt{d}}~.
% O\left(\frac{1}{\sqrt{d}}\right)~.
\]
\end{corollary}

% As our data samples $x \in X$ are assumed to have $\norm{x} = O(\sqrt{d})$, the distance between two input points $x, y \in X$ such that $\text{sign}(f(x)) \neq \text{sign}(f(y))$ is also $O(\sqrt{d})$. In order to allow it, the average gradient along the line $\alpha x + (1= \alpha) y$ for $\alpha \in (0,1)$ should be at least $O(\frac{1}{\sqrt{d}})$. Adversarial examples are defined to be much closer examples from to opposite class (i.e., within $o(\sqrt{d})$ distance from $x$), therefore we expect a much larger gradient at $x$ for adversarial examples to exist.
The proof follows directly from \thmref{thm:upper bound small init}, by noticing that $k\leq m$ and $\ell \leq d$.
Consider for example an input $x_0 \in P$ with $\norm{x_0}=\Theta(\sqrt{d})$, and suppose that $N(x_0) = \Omega(1)$. The above corollary implies that if the initialization has a variance of $1/d^2$ (rather than the standard choice of $1/d$) then the gradient is of size $O\left(1/\sqrt{d}\right)$. Thus, it corresponds to perturbations of size $\Omega(\sqrt{d})$, which is the same order as $\norm{x_0}$.

\subsection{$L_2$ Regularization}
We consider another way to influence the projection onto $P^\perp$ of the trained weights vectors: adding $L_2$ regularization while training. We will update the logistic loss function by adding an additive factor $\frac{1}{2}\lambda\norm{\bw_{1:m}}^2$. 
% For a training data point $x$ labeled as $y$, we define the loss function w.r.t. the weights $\bw_{1:m}$ as $L_R(N(x,\bw_{1:m})\cdot y, \bw_{1:m}) 
% $\log(1+e^{-N(x,\bw_{1:m})\cdot y}) + \lambda \norm{\bw_{1:m}}^2$.
For a dataset $(x_1,y_1),..,(x_r,y_r)$, we now train over the following objective:

 \[ \sum_{j=1}^r L(y_j \cdot N(x_j,\bw_{1:m}))) + \frac{1}{2}\lambda \norm{\bw_{1:m}}^2~.\]
This regularization will cause the previously untrained weights to decrease in each training step which will decrease the upper bound on the projection of the gradient:

\begin{theorem}
\label{thm: explicit regularization gradient}
    Suppose that a network $N(x) = \sum\limits_{i=1}^{m} u_i \sigma(\inner{w_i, x})$ is trained for $T$ training steps, using $L_2$ regularization with parameter $\lambda \geq 0$ and step size $\eta > 0$, on a dataset which lies on a linear subspace $P\subseteq \reals^d$ of dimension $d-\ell$ for $\ell \geq 1$, starting from standard initialization (i.e., $w_i \sim \mathcal{N}(\zero,\frac{1}{d} I_d)$). Let $x_0 \in P$, let $S= \{i\in[m]: \inner{w_i,x_0} \geq 0\}$, and let $k := |S|$.
    %Let the network $N(x) = \sum\limits_{i=1}^{m} u_i \sigma(\inner{w_i,x}) $, for $w_i$ initialized from $\mathcal{N}(0,\frac{1}{d} I_d)$, trained for $T$ training steps with logistic loss using $L_2$ regularization with parameter $\lambda$ and step size bounded by $\eta$, on dataset $X$ which lies on linear subspace $P \subset \reals^d$ of dimension $d-\ell$ for $\ell \geq 1$ , and let $x_0 \in P$. Let $S= \{i\in[m]: \inner{w_i,x_0} \geq 0\}$, and let $k := |S|$.
    Then, w.p. $\geq 1 - e^{-\frac{\ell}{16}}$ we have
    \[\norm{\Pi_{P^\perp}\left(\frac{\partial N(x_0)}{\partial x}\right)} \leq (1- \eta \lambda)^T \sqrt{\frac{2 k \ell}{ m d}}~. \]
\end{theorem}

The full proof can be found in \appref{apen: explicit regularization}. The main idea of the proof is to observe the projection of the weights on $P^\perp$ changing during training. As before, we 
%rotate the data to lie on $M:=\text{span}\{e_1,\dots,e_{d-\ell}\}$ and the weights accordingly.
assume w.l.o.g. that $P = \text{span}\{e_1,\dots,e_{d-\ell}\}$ and denote by $\hat{w}_i:= \Pi_{P^\perp}(w_i)$. 
During training, the weight vector's last $\ell$ coordinates are only affected by the regularization term of the loss. These weights decrease in a constant multiplicand of the previous weights. Thus, we can conclude that for every $t \geq 0$ we have:
$\hat{w}^{(t)}_i = (1- \eta \lambda)^t \hat{w}^{(0)}_i$,  
where $\hat{w}^{(t)}_i$ is the $i$-th weight vector at time $t$.
It implies that our setting is equivalent to initializing the weights with standard deviation $\frac{(1- \eta \lambda)^T}{\sqrt{d}}$ and training the model without regularization for $T$ steps. As a result, we get the following corollary:

\begin{corollary}
For $(1- \eta \lambda)^T \leq \frac{1}{\sqrt{2d}}$, in the settings of \ref{thm: explicit regularization gradient}, w.p. $\geq 1 - e^{-\frac{\ell}{16}}$ we get that:

\[\norm{\Pi_{P^\perp}\left(\frac{\partial N(x_0)}{\partial x}\right)} \leq \frac{1}{\sqrt{d}}~. \]
\end{corollary}

\subsection{Experiments}\label{sec: experiments}
% In this section, we present our robustness-improving experiments. We explore our robust improving methods on two datasets: (1)  A 7-point dataset on a one-dimensional linear subspace in a two-dimensional input space, and (2) A 25-point dataset on a two-dimensional linear subspace in a three-dimensional input space. In Figures \ref{fig:1d data}, \ref{fig:2d data} and \ref{fig:2d data def init}  we present the boundary of a two-layer ReLU network trained over these two datasets. We train the networks until reaching a constant positive margin. 
% We note that unlike our theoretical analysis, in the experiments in Figures \ref{fig:1d data}, \ref{fig:2d data def init} we trained all layers and initialize the weights using the default PyTorch initialization, to verify that the observed phenomena occur also in this setting. In the experiment in \figref{fig:2d data} we use a different initialization scale for the improving effect to be smaller and visualized easily. The experiment in \figref{fig:2d data def init} is demonstrating an extreme robustness effect, accruing when using the standard settings.
% In Figures \ref{fig: 1d clean} \ref{fig: 2d clean}, and \ref{fig: 2d clean def init} we trained using logistic loss. In Figures \ref{fig: 1d small init}, \ref{fig: 2d small init}, \ref{fig: 2d small init def init} we initialized the weights using an initialization with a smaller variance (i.e., initialization divided by a constant factor). Finally, in Figures \ref{fig: 1d regular}, \ref{fig: 2d regular} , and \ref{fig: 2d regular def init} we train with $L_2$ regularization.

In this section, we present our robustness-improving experiments.
\footnote{For the code of the experiments see \url{https://github.com/odeliamel/off-manifold-robustness}}
We explore our 
%robust improving 
methods on two datasets: (1)  A 7-point dataset on a one-dimensional linear subspace in a two-dimensional input space, and (2) A 25-point dataset on a two-dimensional linear subspace in a three-dimensional input space. In Figures \ref{fig:1d data} and \ref{fig:2d data} we present the boundary of a two-layer ReLU network trained over these two datasets. We train the networks until reaching a constant positive margin. 
We note that unlike our theoretical analysis, in the experiments in Figure \ref{fig:1d data} we trained all layers and initialize the weights using the default PyTorch initialization, to verify that the observed phenomena occur also in this setting. In the experiment in \figref{fig:2d data} we use a different initialization scale for the improving effect to be smaller and visualized easily. 
In Figures \ref{fig: 1d clean} and \ref{fig: 2d clean} we trained with default settings. In Figures \ref{fig: 1d small init} and \ref{fig: 2d small init} we initialized the weights using an initialization with a smaller variance (i.e., initialization divided by a constant factor). Finally, in Figures \ref{fig: 1d regular} and \ref{fig: 2d regular} we train with $L_2$ regularization.

% \begin{figure}[htbp]
\begin{figure}[!ht]
\centering
\begin{subfigure}[b]{0.3\textwidth}
    \centering
     \includegraphics[width=\textwidth]{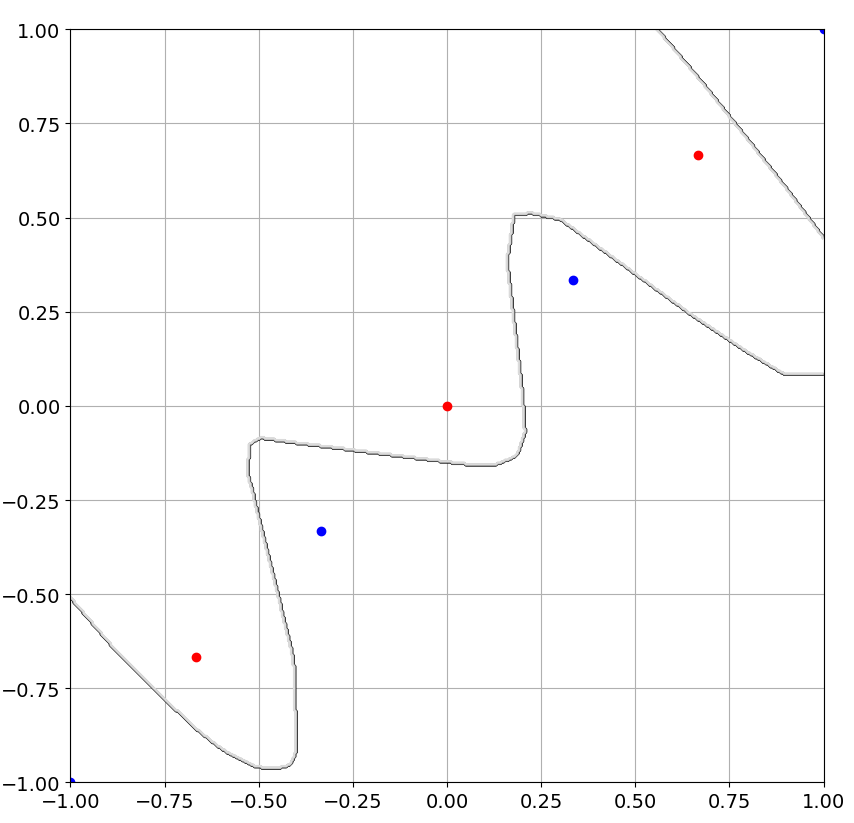}
     \caption{}
     \label{fig: 1d clean}
\end{subfigure}
\begin{subfigure}[b]{0.3\textwidth}
    \centering
     \includegraphics[width=\textwidth]{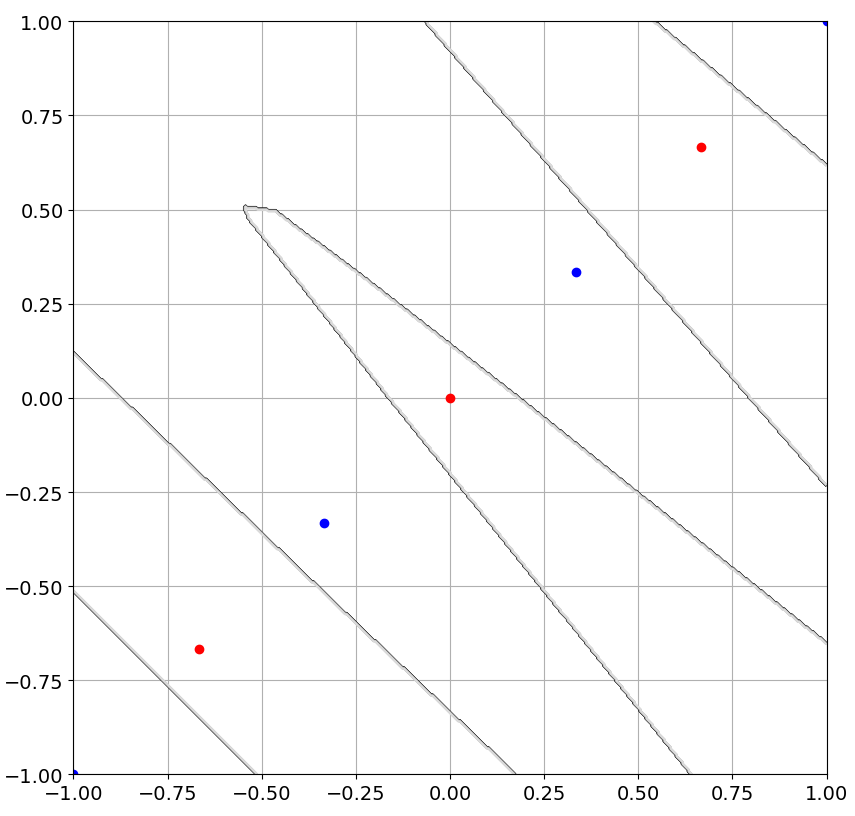}
     \caption{}
     \label{fig: 1d small init}
\end{subfigure}
\begin{subfigure}[b]{0.3\textwidth}
    \centering
     \includegraphics[width=\textwidth]{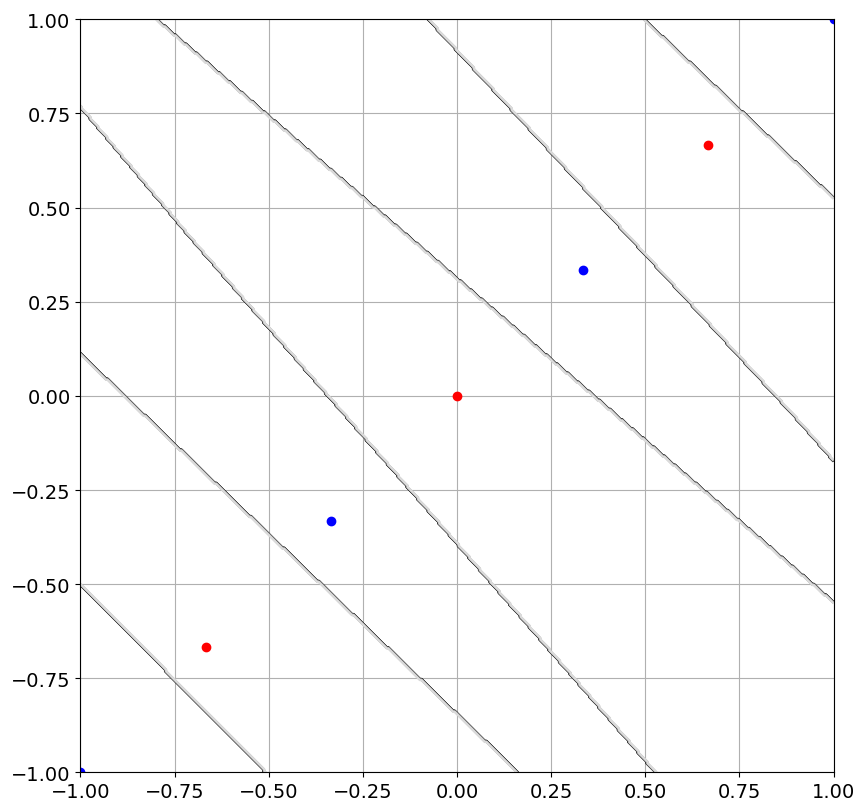}
     \caption{}
     \label{fig: 1d regular}
\end{subfigure}
% \subfigure[]{\includegraphics[width=0.3\textwidth]{imgs/subfigures/defultinit/2layers-clean.png}\label{fig: 1d clean}}
% \subfigure[] {\includegraphics[width=0.3\textwidth]{imgs/subfigures/defultinit/2layers-smallinit.png}\label{fig: 1d small init}}
% \subfigure[]{\includegraphics[width=0.3\textwidth]{imgs/subfigures/defultinit/2layers-regular.png}\label{fig: 1d regular}}
\caption{\textbf{Experiments on a one-dimensional dataset.} We plot the dataset points and the decision boundary in 3 settings: (a) Vanilla trained network, (b) The network's weight are initialized from a smaller variance distribution, and (c) Training with regularization.} \label{fig:1d data}
\end{figure}

Consider the adversarial perturbation in the direction $P^\perp$, orthogonal to the data subspace, in Figures \ref{fig:1d data} and \ref{fig:2d data}. In figure (a) of each experiment, we can see that a rather small adversarial perturbation is needed to cross the boundary in the subspace orthogonal to the data. In the middle figure (b), we see that the boundary in the orthogonal subspace is much further. This is a direct result of the projection of the weights onto this subspace being much smaller. In the right experiment (c), we can see a similar effect created by regularization. 
% In \figref{fig:2d data} we present a milder robustness effect for easier visualization.
%One can see that the dimples \figref{fig:2d data} (b) and (c) are getting deeper when using small initialization or regularization rather than regular training, which demonstrates a boundary further away from the data subspace. 
In \appref{appen: experiments} we add the full default-scaled experiment in the two-dimensional setting to demonstrate the robustness effect. There, in both the small-initialization and regularization experiments, the boundary lines are almost orthogonal to the data subspace.
In \appref{appen: experiments} we also conduct further experiments with deeper networks and standard PyTorch initialization, showing that our theoretical results are also observed empirically in settings going beyond our theory.

% \begin{figure}[htbp]
\begin{figure}[!ht]
\centering
\begin{subfigure}[b]{0.3\textwidth}
    \centering
     \includegraphics[width=\textwidth]{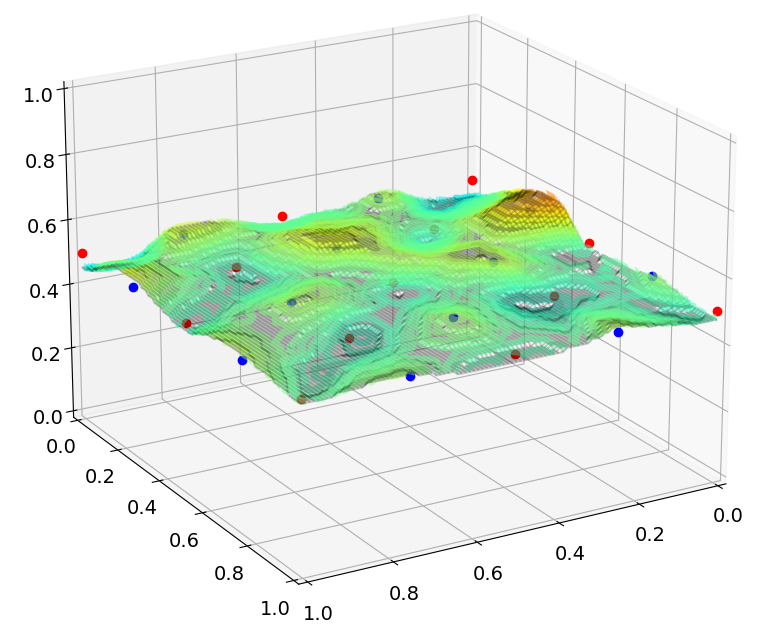}
     \caption{}
     \label{fig: 2d clean}
\end{subfigure}
\begin{subfigure}[b]{0.3\textwidth}
    \centering
     \includegraphics[width=\textwidth]{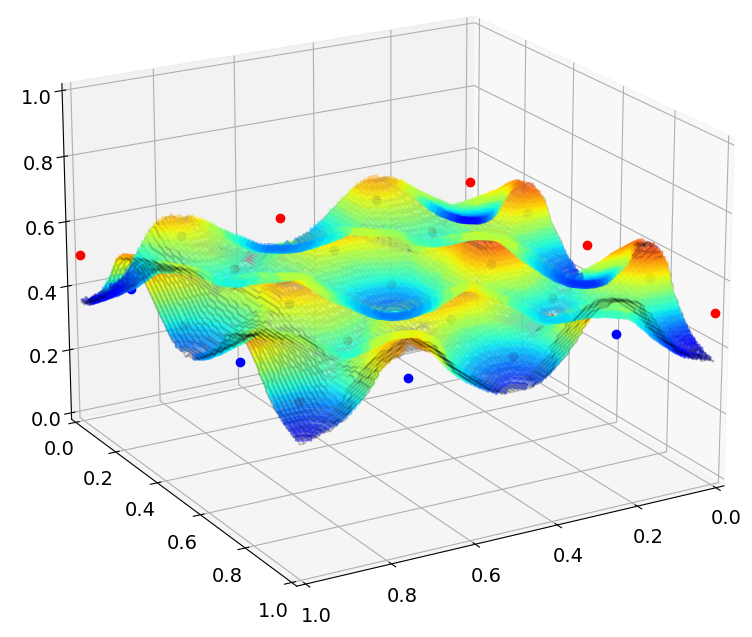}
     \caption{}
     \label{fig: 2d small init}
\end{subfigure}
\begin{subfigure}[b]{0.3\textwidth}
    \centering
     \includegraphics[width=\textwidth]{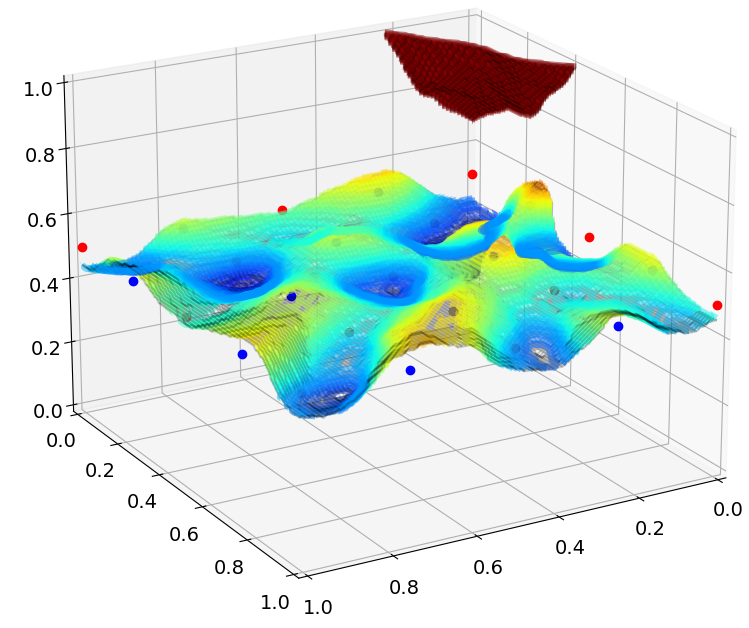}
     \caption{}
     \label{fig: 2d regular}
\end{subfigure}
% \subfigure[]{\includegraphics[width=0.3\textwidth]{imgs/subfigures/chess-clean.png}\label{fig: 2d clean}}
% \subfigure[] {\includegraphics[width=0.3\textwidth]{imgs/subfigures/chess-smallinit.png}\label{fig: 2d small init}}
% \subfigure[]{\includegraphics[width=0.3\textwidth]{imgs/subfigures/chess-ragular.png}\label{fig: 2d regular}}
\caption{\textbf{Experiments on two-dimensional dataset demonstrating a smaller robustness effect.} We plot the dataset points and the decision boundary in 3 settings: (a) Vanilla trained network, (b) The network's weights are initialized from a smaller variance distribution, and (c) Training with regularization. Colors are used to emphasise the values in the $z$ axis.} \label{fig:2d data}
\end{figure}

In \figref{fig:distance boundary} we plot the distance from the decision boundary for different initialization scales of the first layer.  We trained a $3$-layer network, initialized using standard initialization except for the first layer which is divided by the factor represented in the $X$-axis. After training, we randomly picked $200$ points and used a standard projected gradient descent adversarial attack to change the label of each point, which is described in the $Y$-axis (perturbation norm, with error bars). The datasets are: (a) Random points from a sphere with $28$ dimensions, which lies in a space with $784$ dimensions; and (b) MNIST, where the data is projected on $32$ dimensions using PCA. The different lines are adversarial attacks projected either on data subspace, on its orthogonal subspace, or without projection. It can be seen that small initialization increases robustness off the data subspace, and also on the non-projected attack, while having almost no effect for the attacks projected on the data subspace.

\begin{figure}[!ht]
\centering
\begin{subfigure}[b]{0.45\textwidth}
    \centering
     \includegraphics[width=\textwidth]{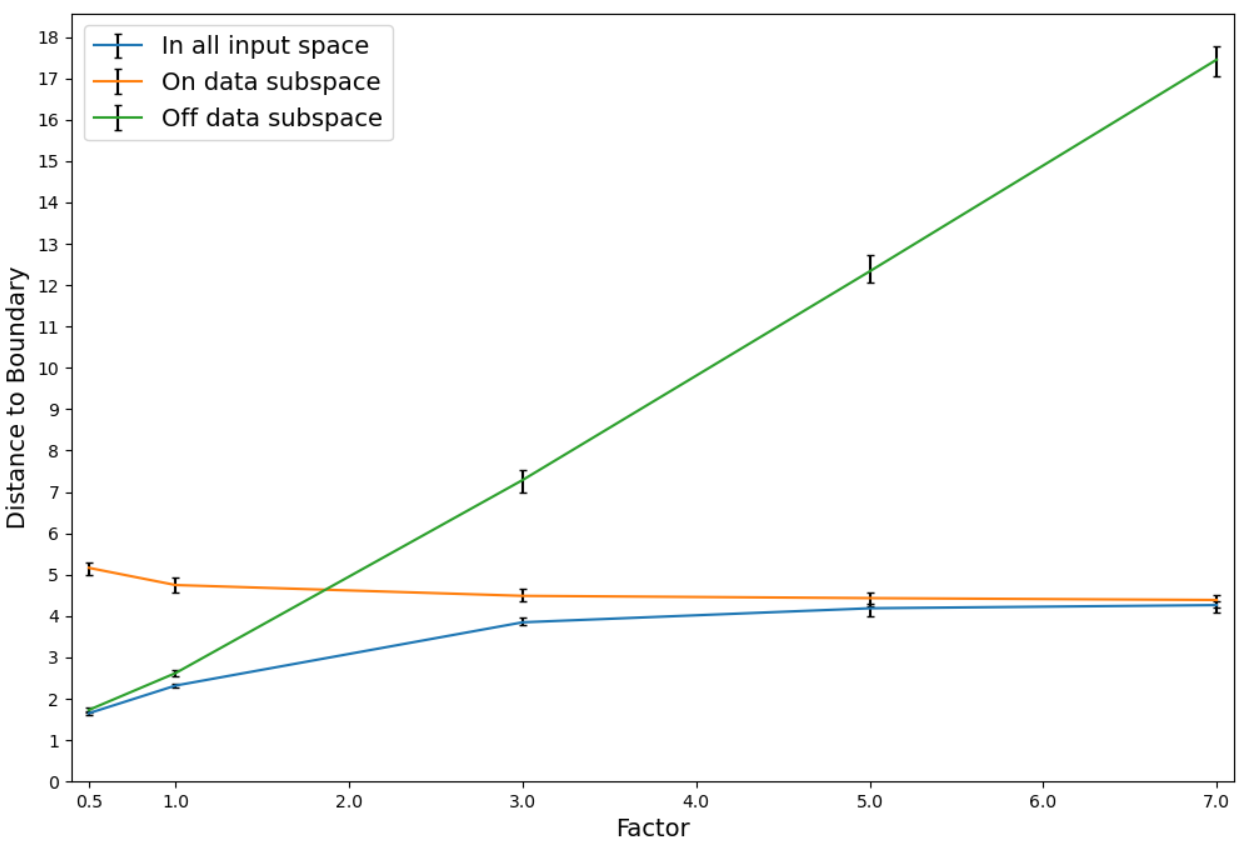}
     \caption{}
     % \label{fig: cifar}
\end{subfigure}
\begin{subfigure}[b]{0.45\textwidth}
    \centering
     \includegraphics[width=\textwidth]{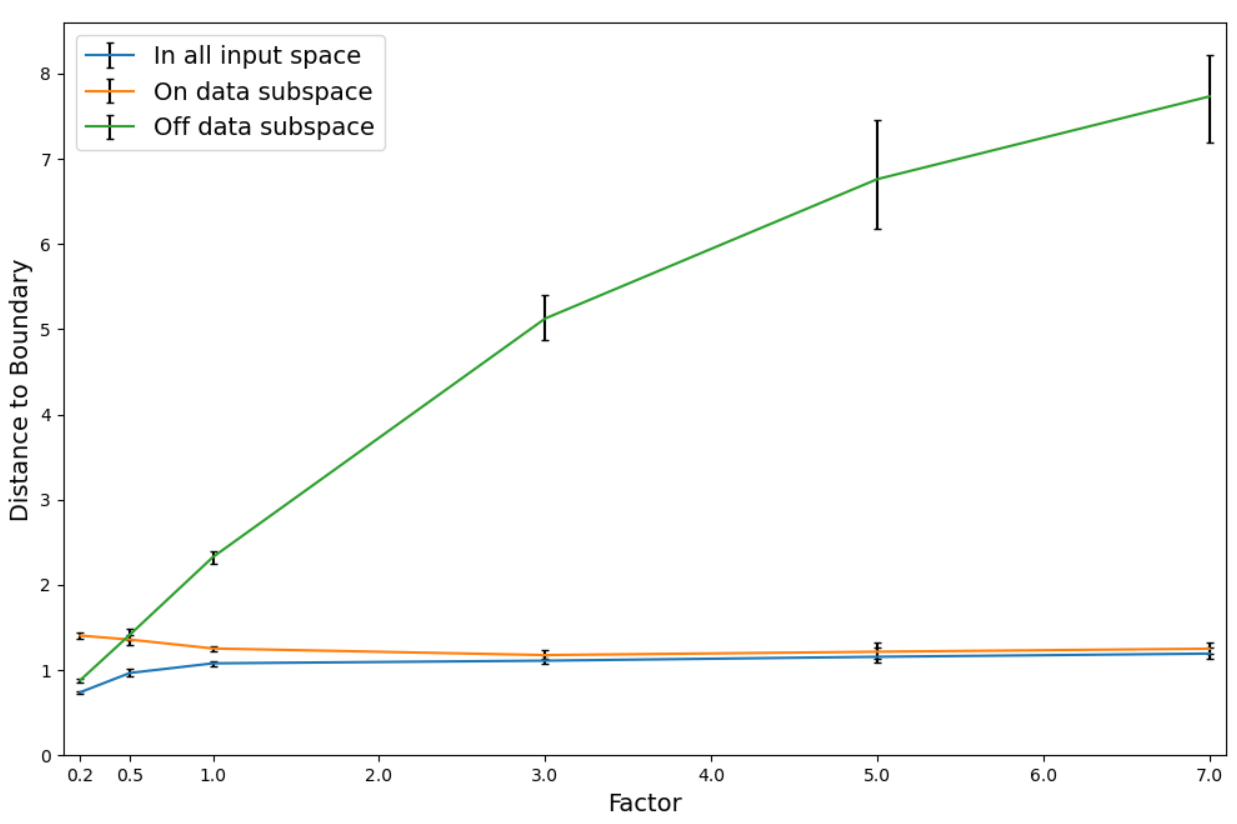}
     \caption{}
     % \label{fig: mnist}
\end{subfigure}
\caption{\textbf{The distance to the decision boundary for different initializations of the first layer.} The $X$-axis represents the factor which the initialization of the first layer is divided by. The $Y$-axis shows the size of a standard perturbed gradient descent adversarial attack to change the label of each point for $200$ randomly picked points.
The datasets are: (a) Random points from a sphere with $28$ dimensions, which lies in a space with $784$ dimensions; and (b) MNIST, where the data is projected on $32$ dimensions using PCA.  The different lines are adversarial attacks projected either on data subspace, on its orthogonal subspace, or without projection. 
% We trained a $3$-layer network, initialized using standard initialization except for the first layer which is divided by the factor represented in the $X$-axis. After training, we randomly picked $200$ points and used a standard projected gradient descent adversarial attack to change the label of each point, which is described in the $Y$-axis (perturbation norm, with error bars). The datasets are: (a) Random points from a sphere with $28$ dimensions, which lies in a space with $784$ dimensions; and (b) MNIST, where the data is projected on $32$ dimensions using PCA. The different lines are adversarial attacks projected either on data subspace, on its orthogonal subspace, or without projection. It can be seen that small initialization increases robustness off the data subspace, and also on the non-projected attack, while having almost no effect for the attacks projected on the data subspace.
}\label{fig:distance boundary}
\end{figure}
\section{Conclusions and Future Work}\label{sec:conclusions}
In this paper we considered training a two-layer network on a dataset lying on $P\subseteq \reals^d$ where $P$ is a $d-\ell$ dimensional subspace. We have shown that the gradient of any point $x_0\in P$ projected on $P^\perp$ is large, depending on the dimension of $P$ and the fraction of active neurons on $x_0$. We additionally showed that there exists an adversarial perturbation in the direction of $P^\perp$. The size of the perturbation depends in addition on the output of the network on $x_0$, which by Remark \ref{remark:margin} should be poly-logarithmic in $d$, at least for points which lie on the margin of the network. Finally, we showed that by either decreasing the initialization scale or adding $L_2$ regularization we can make the network robust to ``off-manifold'' perturbations, by decreasing the gradient in this direction.

One interesting question is whether our results can be generalized to other manifolds, beyond linear subspaces. We state this as an informal open problem:
% We conjecture that for general low-dimensional manifolds $M\subseteq \reals^d$, if we train a network on a dataset which lies on $M$, then for every $x_0\in M$ there exists an adversarial example in the direction of $\left(T_{x_0}M\right)^\perp$, i.e. orthogonal to the tangent space where $T_{x_0} M$, of $M$ at $x_0$.

\begin{open problem}
    Let $M$ be a manifold, and $\mathcal{D}$ a distribution over $M\times\{\pm 1\}$. Suppose we train a network $N:\reals^d\rightarrow\reals$ on a dataset sampled from $\mathcal{D}$. Let $x_0\in M$, then under what conditions on $M,~\mathcal{D}$ and $N$, there exists a small adversarial perturbation in the direction of $\left(T_{x_0}M\right)^\perp$, i.e. orthogonal to the tangent space $T_{x_0} M$, of $M$ at $x_0$.
\end{open problem}
% they can:
% \begin{conjecture}\label{con:manifold}
%     Let $M\subseteq \reals^d$ be a $d-\ell$ dimensional manifold where $\ell = \Theta(d)$. Suppose we train a two-layer ReLU network $N:\reals^d\rightarrow \{\pm 1\}$ on a dataset which lies on $M$. For a point $x\in M$ we denote by $T_x M$ the tangent space of $M$ at x. Then, for every $x_0\in M$ with $\norm{x_0} =O(\sqrt{d})$ there exists an adversarial perturbation $z$ in the direction of $\left(T_{x_0}M\right)^\perp$ such that $\norm{z} = o(\sqrt{d})$ and $\text{sign}(N(x_0 + z)) \neq \text{sign}(N(x_0))$
%  \end{conjecture}
Our result can be seen as a special case of this conjecture, where at all points $x,x'\in M$, the tangent spaces are equal $T_xM = T_{x'}M$. Another future direction would be to analyze deeper networks, or different architectures such as convolutions. Finally, it would also be interesting to analyze robustness of trained networks w.r.t. different norms such as $L_1$ or $L_\infty$.

%\subsection*{Acknowledgments}

% % Acknowledgments---Will not appear in anonymized version
\subsection*{Acknowledgments}
We thank Ohad Shamir for the many helpful discussions about this work. We would also like to thank Michal Irani for contributing computational resources. GY was supported in part by the European Research Council (ERC) grant 754705 . GV acknowledges the support of the NSF and the Simons Foundation for the Collaboration on the Theoretical Foundations of Deep Learning.
% \acks{We thank a bunch of people and funding agency.}
%\newpage
\bibliographystyle{plainnat}
\bibliography{bib}

\newpage

\appendix
\newcommand*{\horzbar}{\rule[.5ex]{2.5ex}{0.5pt}}
\newcommand*{\vertbar}{\rule[-1ex]{0.5pt}{2.5ex}}

\section{Rotation Invariance w.r.t. the Initialized Weights}

In this paper, we analyze neural networks trained on high-dimensional data that lies on a low dimensional linear subspace denoted by $P$. We assume that the dimension of $P$ is $d-\ell$. Throughout the paper it will be more convenient to analyze data which lies on the subspace $M=\text{span}(\{e_1,\ldots,e_{d-\ell}\})$, because then the ``off manifold'' directions correspond exactly to certain coordinates of the input. In this section we show that we can essentially analyze the data as if it is rotated to lie on $M$, and it would imply the same consequences as the original data from $P$.

% Particularly, for an input dimension $d$, and a lower data-dimension $d-\ell$, we take data points which lies on the linear manifold $M$ for $M \subset \reals^d $ spanned by $\{e_1,...,e_{d-\ell}\}$. In this section we show that one can use this simplified manifold instead of any $d-\ell$-dimensional linear manifold thanks to basic properties of orthogonal transformations and Gaussian distribution.

\begin{theorem}\label{thm:rotation invariant}
    Let $P\subseteq \reals^d$ be a subspace of dimension $d-\ell$, and let $M=\text{span}\{e_1,\dots,e_{d-\ell}\}$. Let $R$ be an orthogonal matrix such that $R\cdot P = M$, let $X\subseteq P$ be a training dataset and let $X_R = \{R\cdot x: x\in X\}$. Assume we train a neural network $N(x)= \sum_{i=1}^m u_i\sigma(w_i^\top x)$ as explained in \secref{frameworksection}, and denote by $N^X$ and $N^{X_R}$ the network trained on $X$ and $X_R$ respectively for the same number of iterations. Let $x_0\in P$, then we have:
    \begin{enumerate}
        \item W.p. $p$ (over the initialization) we have $\left\|\Pi_{P^\perp}\left(\frac{\partial N^X(x_0)}{\partial x} \right)\right\| \geq c$ (resp. $\leq c$) for some $c\in \reals$, iff w.p. $p$ also $\left\|\Pi_{M^\perp}\left(\frac{\partial N^{X_R}(Rx_0)}{\partial x} \right)\right\| \geq c$ (resp. $\leq c$).
        %\item For any $z\in\reals^d$ and $c\in\reals$, w.p. $p$  (over the initialization) we have $N^X(x_0 + z) - N^X(x_0) \geq c$ (resp. $\leq c$), iff w.p. $p$ also $N^{X_R}(Rx_0 + Rz)  - N^{X_R}(Rx_0)\geq c$ (resp. $\leq c$).
        \item For any $c,p \geq 0$, w.p. $p$ (over the initialization) there exists $z \in P^\perp$ with $\norm{z}=c$ such that $\sign(N^X(x_0 + z)) \neq \sign(N^X(x_0))$, iff w.p. $p$ there exists $z' \in M^\perp$ with $\norm{z'}=c$ such that $\sign(N^{X_R}(Rx_0 + z')) \neq \sign(N^{X_R}(Rx_0))$.
    \end{enumerate}
\end{theorem}

\begin{proof}
    Denote by $\bw_{1:m}:= (w_1,\dots,w_m)$ and by $R\bw_{1:m} = (Rw_1,\dots,Rw_m)$.  Let $\bw_{1:m}^{(t)}$ the weights of the network trained on the dataset $X$ where $\bw_{1:m}^{(0)}$ is some initialization, and $\tilde{\bw}_{1:m}^{(t)} =(\tilde{w}_1^{(t)},\dots,\tilde{w}_m^{(t)})$ the weights of the network trained on $X_R$ and initialized at $R\bw_{1:m}^{(0)}$. In the proof, when taking derivatives w.r.t. the $w_i$'s we will explicitly write $N(x,\bw_{1:m})$.
    
    % and by $N^X_{\bw_{1:m}}$ a network trained on dataset $X$ with weights initialized at $\bw_{1:m}$. 
    % We also denote by $\bw_{1:m}^{(t)}$ the weights of the network $N^X_{\bw_{1:m}^{(0)}}$, where $\bw_{1:m}^{(0)}$ are some initialized weights, and by $\tilde{\bw}_{1:m}^{(t)} =(\tilde{w}_1^{(t)},\dots,\tilde{w}_m^{(t)})$ the weights of $N^{X_R}_{R\bw_{1:m}^{(0)}}$ trained after $t$ steps, where $R\bw_{1:m} = (Rw_1,\dots,Rw_m)$. In the proof, when taking derivatives w.r.t the $w_i$'s we will also explicitly write $N^{X_R}_{R\bw_{1:m}^{(0)}}(x, \bw_{1:m})$, that is that $N$ also depends on the weights $\bw_{1:m}$.
    
    We first show by induction on the number of training steps that $\tilde{\bw}_{1:m}^{(t)} = R{\bw}_{1:m}^{(t)}$. For $t=0$ it is clear by the assumption on the initialization. Assume it is true for $t$, then we have for some $x\in X$:
     \begin{align*}
        \frac{\partial N(Rx, \tilde{\bw}_{1:m}^{(t)})}{\partial w_i} & = u_i \sigma'(\inner{\tilde{w}_i^{(t)}, Rx})Rx \nonumber\\
        & = u_i \sigma'(\inner{Rw_i^{(t)}, Rx})Rx \nonumber\\
        & = u_i \sigma'(\inner{w_i^{(t)}, x})Rx \nonumber\\
        & =R\cdot  \frac{\partial N(x, \bw_{1:m}^{(t)})}{\partial {w}_i}~.%\label{eq:gradients rotation}
    \end{align*}   
    % \begin{align}
    %     \frac{\partial N^{X_R}_{R\bw_{1:m}^{(0)}}(Rx, \tilde{\bw}_{1:m}^{(t)})}{\partial w_i} & = u_i \sigma'(\inner{\tilde{w}_i^{(t)}, Rx})Rx \nonumber\\
    %     & = u_i \sigma'(\inner{Rw_i^{(t)}, Rx})Rx \nonumber\\
    %     & = u_i \sigma'(\inner{w_i^{(t)}, x})Rx \nonumber\\
    %     & =R\cdot  \frac{\partial N^{X}_{\bw_{1:m}^{(0)}}(x, \bw_{1:m}^{(t)})}{\partial {w}_i}\label{eq:gradients rotation}
    % \end{align}
    This is true for every $i\in[m]$ and for every $x\in X$. Also note that by our induction assumption we have:
    \begin{align}\label{eq:equal after rotation}
        N(x, \bw_{1:m}^{(t)}) = \sum_{i=1}^m u_i\sigma(\inner{w_i^{(t)}, x}) = \sum_{i=1}^m u_i\sigma(\inner{Rw_i^{(t)}, Rx}) = N(Rx, \tilde{\bw}_{1:m}^{(t)})~.        
    \end{align}
    % \begin{align}\label{eq:equal after rotation}
    %     N^X_{\bw_{1:m}^{(0)}}(x, \bw_{1:m}^{(t)}) = \sum_{i=1}^m u_i\sigma(\inner{w_i^{(t)}, x}) = \sum_{i=1}^m u_i\sigma(\inner{Rw_i^{(t)}, Rx}) = N^{X_R}_{R\bw_{1:m}^{(0)}}(Rx, \bw_{1:m}^{(t)})        
    % \end{align}
    Finally, the derivative of the loss on a single data point $x\in X$ with label $y$ can be written as:
    \[
    \frac{\partial L\left( N(x, \bw_{1:m}^{(t)})\cdot y\right)}{\partial w_i} =
    L'\left( N(x, \bw_{1:m}^{(t)})\cdot y\right)\cdot \frac{\partial N(x, \bw_{1:m}^{(t)})}{\partial w_i}~,
    \]
    % \[
    % \frac{\partial \ell\left( N^X_{\bw_{1:m}^{(0)}}(x, \bw_{1:m}^{(t)}),y\right)}{\partial w_i} = \frac{\partial \ell\left( N^X_{\bw_{1:m}^{(0)}}(x, \bw_{1:m}^{(t)}),y\right)}{\partial N^X_{\bw_{1:m}^{(0)}}} \cdot \frac{\partial N^X_{\bw_{1:m}^{(0)}}(x, \bw_{1:m}^{(t)})}{\partial w_i}
    % \]
    where the first term depends only on the value of $N(x, \bw_{1:m}^{(t)})$. Hence, taking a single gradient step of $N$ with weights $\bw_{1:m}^{(t)}$ and dataset $X$
    % with $N^{X}_{\bw_{1:m}^{(0)}}$ on some $x\in X$ 
    will change the weights by the same term up to multiplication by $R$ as if taking a gradient step with with weights $\tilde{\bw}_{1:m}^{(t)}$ and dataset $X_R$. This finishes the induction.
    
    Let $\bw_{1:m}^{(0)}$ be an initialization for the training of $N^X$, where there exists $z \in P^\perp$ with $\norm{z}=c$ such that $\sign(N^X(x_0 + z)) \neq \sign(N^X(x_0))$. Then, by \eqref{eq:equal after rotation} the initialization $R \bw_{1:m}^{(0)}$ for the training of $N^{X_R}$ is such that for $z' = R z$ we have $\norm{z'}=c$ and $\sign(N^{X_R}(Rx_0 + z')) \neq \sign(N^{X_R}(Rx_0))$. This argument holds also in the opposite direction. 
    Let $A\subseteq \{\bw_{1:m}\in \reals^{d\cdot m}\}$ be the set of all initializations to $N^X$ where there exists $z \in P^\perp$ with $\norm{z}=c$ such that $\sign(N^X(x_0 + z)) \neq \sign(N^X(x_0))$, then by the above the set $R\cdot A = \{R\bw_{1:m}: \bw_{1:m}\in A\}$ are exactly all the initializations to $N^{X_R}$ where there exists $z' \in M^\perp$ with $\norm{z'}=c$ such that $\sign(N^{X_R}(Rx_0 + z')) \neq \sign(N^{X_R}(Rx_0))$. Since we initialize the $w_i$'s using a Gaussian initialization which is spherically symmetric, we have that $\Pr(A) = \Pr(RA)$. This proves item (2).
    Item (1) follows from similar arguments (which we do not repeat for conciseness).
\end{proof}

Under the assumption that the data lies on $M=\text{span}\{e_1,\dots,e_{d-\ell}\}$, and no regularization is used, we can show that the weights of the first layer projected on $M^\perp$ do not change during training. This is an essential part of the proofs, as it allows us to analyze those weights as random Gaussian vectors, and apply concentration bounds on them.

\begin{theorem}\label{thm:weights dont change}
    Let $M=\text{span}\{e_1,\dots,e_{d-\ell}\}$. Assume we train a neural network $N(x, \bw_{1:m}) := \sum_{i=1}^m u_i\sigma(w_i^\top x)$ as explained in \secref{frameworksection} (where $\bw_{1:m} = (w_1,\dots,w_m)$). Denote by $\hat{w}:= \Pi_{M^\perp}(w)$ for $w\in\reals^d$, then after training, for each $i\in[m]$, $\hat{w}_i$ did not change from their initial value.
\end{theorem}

\begin{proof}
    Note that for each $i\in [m]$ and $x\in M$ we have:
    \begin{align*}
         \Pi_{M^\perp}\left(\frac{\partial N(x,\bw_{1:m})}{\partial w_i}\right) = \Pi_{M^\perp}\left(u_i\sigma'(w_i^\top x) x\right) = u_i \sigma'(w_i^\top x)\hat{x} = \bm{0}~.
    \end{align*}
    Taking the derivative of the loss we have:
    % \begin{align*}
    % \Pi_{M^\perp}\left(\frac{\partial \ell\left( N(x,\bw_{1:m}), y\right)}{\partial w_i}\right) &= \Pi_{M^\perp}\left(\frac{\partial \ell\left( N(x,\bw_{1:m}), y\right)}{\partial N}\cdot \frac{\partial  N(x,\bw_{1:m})}{\partial w_i}\right)
    % \\
    % & = \frac{\partial \ell\left( N(x,\bw_{1:m}), y\right)}{\partial N}\cdot\Pi_{M^\perp}\left(\frac{\partial  N(x,\bw_{1:m})}{\partial w_i}\right) = \bm{0}~.
    % \end{align*}
    \begin{align*}
    \Pi_{M^\perp}\left(\frac{\partial L\left( N(x,\bw_{1:m})\cdot y\right)}{\partial w_i}\right) &= \Pi_{M^\perp}\left(L'\left( N(x,\bw_{1:m})\cdot y\right)\cdot \frac{\partial  N(x,\bw_{1:m})}{\partial w_i}\right)
    \\
    & = L'\left( N(x,\bw_{1:m})\cdot y\right)\cdot\Pi_{M^\perp}\left(\frac{\partial  N(x,\bw_{1:m})}{\partial w_i}\right) = \bm{0}~.
    \end{align*}
    The above calculation did not depend on the specific value of the $w_i$'s. Hence, the value of the $\hat{w}_i$'s for every $i\in[m]$ did not change  during training from their initial value.
\end{proof}

\section{Proofs from \secref{sec:large gradient}}\label{appen:proofs from large gradient}

Before proving the main theorem, we will first need the next two lemmas about the concentration of Gaussian random variables:

\begin{lemma}
\label{randomnormlarge}

    Let $w\in\reals^n$ such that $w \sim \mathcal{N}(\zero,\sigma^2 I_n)$. Then:

    % \begin{enumerate}
    %     \item $\mathbb{E} \left[\norm{w}^2\right] =  \sigma^2 n$ ($= \frac{n}{c}$)
        \[ \mathbb{P} \left[\norm{w}^2 \leq \frac{1}{2} \sigma^2 n \right] \leq e^{-\frac{n}{16}}~.\] % ($\mathbb{P} \left[\norm{w}^2 \leq \frac{1}{2} \frac{n}{c} \right] \leq e^{-\frac{n}{16}}$)
    % \end{enumerate}
\end{lemma}

\begin{proof}
	Note that $\norm{\frac{w}{\sigma}}^2$ has the Chi-squared distribution. A concentration bound by Laurent and Massart \citep[Lemma 1]{laurent2000adaptive} implies that for all $t > 0$ we have
	\[
	 	\Pr\left[ n - \norm{\frac{w}{\sigma}}^2 \geq 2\sqrt{nt} \right] \leq e^{-t}~.
	\]
	Plugging-in $t=\frac{n}{16}$, we get
	\begin{align*}
        \Pr\left[ n - \norm{\frac{w}{\sigma}}^2 \geq \frac{1}{2} n \right] 
		= \Pr\left[\norm{\frac{w}{\sigma}}^2 \leq \frac{1}{2} n  \right] 
		\leq e^{-n/16}~.
	\end{align*}
	Thus, we have
	\begin{equation*} 
		\Pr\left[ \norm{w}  \leq \sigma \sqrt{\frac{n}{2}} \right] 
		\leq e^{-n/16}~.
	\end{equation*}
\end{proof}

\begin{lemma}
\label{randomlydirected}
    Let $w_1,\ldots,w_m \in\reals^n$ such that for all $i \in [m]$, $w_i \sim \mathcal{N}(\zero,\sigma^2 I_n)$, then we have: %(for $\sigma = \frac{1}{\sqrt{c}})$ i.e. m n-dimensional vectors randomly directed in the n-dimensional sphere. Then - 
    % \begin{enumerate}
    %     \item $\mathbb{E} \left[ \norm{\sum\limits_{i=1}^m w_i}^2 \right] = m \cdot \mathbb{E} \left[  \norm{w_i}^2\right]  = m \sigma^2 n (= \frac{mn}{c} )$
     \[
         \mathbb{P} \left[\norm{\sum\limits_{i=1}^m w_i}^2 \leq \frac{1}{2} m \sigma^2 n \right] \leq e^{-\frac{n}{16}}~. 
     \]
    % \end{enumerate}
\end{lemma}

\begin{proof}
    We denote the $j$-th coordinate of the vector $w_i \in\reals^n$ by $w_{i,j}$. Note, for any $i \in \{1,\ldots,m\}$ and $j \in \{1,\ldots,n\}$ we have $w_{i,j} \sim \mathcal{N}(0,\sigma^2)$. We denote by $s$ the sum vector $s :=\sum\limits_{i=1}^m w_{i}$, and by $s_j$ the $j$- th coordinate of $s$. By this definition, $s_j  = \sum\limits_{i=1}^m w_{i,j}$ is a sum of $m$ independent Gaussian variables and therefore also a Gaussian variable. Particularly, $s \sim \mathcal{N}(\zero,m \sigma^2 I_n)$.  We use Lemma~\ref{randomnormlarge} with variance $m \sigma^2$ and get that:
    % $$\mathbb{E} \left[ \norm{\sum\limits_{i=1}^m w_i}^2 \right] =\mathbb{E} \left[ \norm{s}^2 \right] =  m \sigma^2 n =  m \cdot \mathbb{E} \left[  \norm{w_i}^2\right] $$
    % And - 
    \[\mathbb{P} \left[\norm{\sum\limits_{i=1}^m w_i}^2 \leq \frac{1}{2} m \sigma^2 n \right] \leq e^{-\frac{n}{16}}~.\]

\end{proof}

We are now ready to prove the main theorem of this section:

\begin{proof}[Proof of \thmref{largegradient}]

Let $M = \text{span}\{e_1,\dots,e_{d-\ell}\}$. By \thmref{thm:rotation invariant}(1), given a training dataset $X\subseteq P$, it is enough to consider a training set $X_R = \{Rx: x\in X\}$, where $R$ is an orthogonal matrix such that $R\cdot P = M$, and training is done over $X_R$. From now on, we assume that the training data, as well as $x_0$ lie on $M$, and the consequences of this proof would also imply for a dataset $X$ and $x_0\in P$.

% \paragraph{Data manifold rotation for proof simplification}
% Denote by $M$ the $d-\ell$ dimensional linear manifold spanned by $\{e_1,\ldots,e_{d-\ell}\}$ where $e_i$ are the standard basis vectors. Let $R$ be an orthogonal transformation matrix $R$ given by \lemref{rotatedata} from the manifold $P$ to $M$. We denote by $w^P$ the weights of the network when trained on data $X \subset P$, and by $w^M$ the weights trained on $X_R = \{Rx | x \in X\} \subset M$. From \thmref{trainrotationinvar} we get that $w^M$ and $w^P$ have the same distribution up to the orthogonal transformation $R$. In addition, from \lemref{rotateoffmanifold} we conclude that the projection of the vectors $w^P$ and $w^M$ on the subspace orthogonal to the data manifolds $P$ and $M$ respectively have the same distribution as well, up to $R$. Since $R$ is an orthogonal translation, it preserved norms and angles. Since the following proof only involves these properties of the weight vectors, we get to the same conclusions for $w^P$ and $w^M$. Thus, for simplicity, we continue the proof for $w:= w^M$.

% Note, for $x_0$ that lies on $M$ we get that $\hat{x}_0 = \Pi_{M^\perp}\left(x_0\right)= \bm{0}$.

% By the definition of an active neuron we get 
The projection of the gradient on $M^\perp$ is equal to:
\[
    \Pi_{M^\perp}\left(\frac{\partial N(x_0)}{\partial x}\right) = \Pi_{M^\perp}\left(\sum\limits_{i=1}^{m} u_i w_i \mathbbm{1}_{\inner{w_i,x_0} \geq 0}\right) = \sum\limits_{i=1}^{m} \Pi_{M^\perp}\left(u_i w_i\right)  \mathbbm{1}_{i \in S} =  \sum\limits_{i \in S} \Pi_{M^\perp}\left(u_i w_i\right)~.  
\]
% In this proof, we look at the projection of the gradient on $M^\perp$. Therefore, we
Denote by $\hat{w}_i = (w_i)_{d-\ell+1:d}$, the last $\ell$ coordinates of $w_i$. By \thmref{thm:weights dont change} we get that for every $i\in[m]$, $\hat{w}_i$ did not change from their initial value during training.
% Note that for $j\leq d-\ell$, $\Pi_{M^\perp}\left((w_i)_{j}\right) = 0$, i.e. the first $d-\ell$ coordinates of $\Pi_{M^\perp}(w_i)$ are zero. Using the above notation, we have:
% \[
% \left\|\Pi_{M^\perp}\left(\frac{\partial N(x_0)}{\partial x_0}\right) \right\|
% = \left\|\sum\limits_{i \in S} \hat{w}_i u_i\right\|
% \]

% Under the above notation, we can write\note{continue from here}

% the projection of $w_i$ on $M^\perp$.
% , and note that the projection on $M^\perp$ of the gradient is determined only by them. We denote by $w_{i,j}$ the $j$ coordinate of the weight vector $w_i$. For the data manifold $M$ we get that- 

% \[
%     \hat{w}_{i,j} = 
% \begin{cases}
%     0& \text{if } j\leq d-\ell\\
%     w_{i,j}              & \text{otherwise}
% \end{cases}
% \]\\

% Under these notations, we can write - 

% \[\Pi_{M^\perp}\left(\frac{\partial N(x_0)}{\partial x_0}\right) = \sum\limits_{i \in S} \hat{w}_i u_i  \]

% We will now focus on bounding the l.h.s of the above equation. Note that:
% \[
%     \frac{\partial N(x)}{\partial w_i} = u_i\sigma(w_i^\top x)x
% \]
% and since $\hat{x}= \bm{0}$, this means that $\hat{w}_i$ haven't changed during training from their initial value. 
Recall that we initialized $\hat{w}_i \sim \mathcal{N}(\zero,\frac{1}{\sqrt{d}}I_{\ell})$.
% , and from \ref{trainrotationinvar} the weights haven't changed during training. 
Note that the set $S$ is independent of the value of the $\hat{w}_i$'s. This is because $\hat{w}_i$ does not effect the training, hence will not effect $w_i - \Pi_{M^\perp}(w_i)%szA
$. Also, after choosing $x_0$ we have $\inner{\hat{w}_i,\hat{x}_0} = 0$, since $\hat{x}_0 = \bm{0}$, which means that the choice of $S$ is independent of the $\hat{w}_i$'s. We can conclude that the random variables $\hat{w}_i$ for $i\in S$ are sampled independently.
% and hence does not effect $S$. Thus, we could have sampled the $\hat{w}_i$'s after the training is finished, and the set $S$ would be already determined.
% Therefore, from Lemma~\ref{randomlydirected} we get that w.p. $\geq 1- e^{-\ell/16}$: 

% \[ 
%     \norm{  \sum\limits_{i \in S} \hat{w}_i} \geq \sqrt{\frac{1}{2}} \sqrt{\frac{k l}{d}}~.
% \]
Note, since for all $i \in \{1,\ldots,m\}$, $|u_i| = \frac{1}{\sqrt{m}}$ and they are not trained, we get that $u_i\hat{w}_i$ are also Gaussian random variables with the same mean, and variance multiplied by $\frac{1}{{m}}$. Therefore, from Lemma~\ref{randomlydirected} we get that w.p. $\geq 1- e^{-\ell/16}$: 

\[ 
    \norm{  \sum\limits_{i \in S} u_i\hat{w}_i} \geq \sqrt{\frac{1}{2}} \sqrt{\frac{k l}{dm}}~.
\]
Combining the above, we get:
\[
    \left\|\Pi_{M^\perp}\left(\frac{\partial N(x_0)}{\partial x}\right) \right\| \geq \sqrt{\frac{1}{2}}\sqrt{\frac{kl}{dm}}~.
\]
% and $\hat{w}_i$s distributed symmetrically (i.e. $-\hat{w}_i$ and $\hat{w}_i$ have the same distribution, as its coordinates are either zeros or $ \mathcal{N}(0,\frac{1}{\sqrt{d}})$) we get that for any $r>0$:
%  \[ \Pr\left[ \norm{  \sum\limits_{i \in S} u_i \hat{w}_i} \geq r \right] =  \Pr\left[ \frac{1}{\sqrt{m}} \norm{  \sum\limits_{i \in S}\hat{w}_i} \geq r \right] \] 

% Altogether we have that %w.p. $\geq 1- e^{-\frac{n}{16}}$-

% 	\begin{equation*} 
%   \Pr\left[ \norm{\Pi_{M^\perp}\left(\frac{\partial N(x_0)}{\partial x_0}\right)} \geq  \sqrt{\frac{1}{2}} \sqrt{\frac{k l}{m d}} \right] 
%   =   \Pr\left[ \norm{  \sum\limits_{i \in S} u_i \hat{w}_i} \geq  \sqrt{\frac{1}{2}} \sqrt{\frac{k l}{m d}} \right] 
%   \end{equation*}
%   \begin{equation*} 
%   = \Pr\left[ \frac{1}{\sqrt{m}} \norm{  \sum\limits_{i \in S}\hat{w}_i}  \geq  \sqrt{\frac{1}{2}} \sqrt{\frac{k l}{m d}} \right] \geq 1- e^{-\frac{n}{16}}~. 
%  \end{equation*}
\end{proof}

% \begin{corollary}
% For $\ell = O(d)$, $k = O(m)$ and $m = O(d)$, with probability $\geq 1- e^{-\frac{n}{16}}$ we get that:
% \[ \norm{\Pi_{M^\perp}\left(\frac{\partial N(x_0)}{\partial x_0}\right)}  = O(1) \]
% \end{corollary}

\section{Proofs from \secref{sec:adv pert exists}}\label{appen:proofs from adv perf exists}

Before proving the main theorem, we prove a few lemmas about concentration of Gaussian random variables:

%randomnormsmall
\begin{lemma}
\label{lem:randomnormsmall}
    Let $w \in \reals^n$ with $w \sim \mathcal{N}(\zero,\sigma^2 I_n)$. Then:
        \[
        \Pr \left[\norm{w}^2 \geq 2 \sigma^2 n \right] \leq e^{-\frac{n}{16}}~.
        \] 
\end{lemma}
\begin{proof}
	Note that $\norm{\frac{w}{\sigma}}^2$ has the Chi-squared distribution. A concentration bound by Laurent and Massart \citep[Lemma 1]{laurent2000adaptive} implies that for all $t > 0$ we have
	\[
	 	\Pr\left[ \norm{\frac{w}{\sigma}}^2 - n \geq 2\sqrt{nt} + 2t \right] \leq e^{-t}~.
	\]
	Plugging-in $t=\frac{n}{16}$, we get
	\begin{align*}
        \Pr\left[ \norm{\frac{w}{\sigma}}^2  \geq  2n \right] 
		\leq \Pr\left[ \norm{\frac{w}{\sigma}}^2 - n \geq  n/2 + n/8 \right] 
		\leq e^{-n/16}~.
	\end{align*}
	Thus, we have
	\begin{equation*} 
		\Pr\left[ \norm{w}  \geq \sigma \sqrt{2n} \right] 
		\leq e^{-n/16}~.
	\end{equation*}
\end{proof}

%1randominnermul
\begin{lemma}
\label{lem:randominnermul}
    Let $u \in \reals^n$, and $v \sim \mathcal{N}(\zero,\sigma^2 I_n)$. Then, for every $t>0$ we have
        \item \[\Pr \left[| \inner{u,v}| \geq \norm{u} t  \right] \leq 2 \exp \left(-\frac{t^2}{2\sigma^2}\right)~.\]
\end{lemma}
\begin{proof}
We first consider $\inner{\frac{u}{\norm{u}},v}$. As the distribution $\mathcal{N}(\zero,\sigma^2 I_n)$ is rotation invariant, one can rotate $u$ and $v$ to get $\Tilde{u}$ and $\Tilde{v}$ such that $\Tilde{\frac{u}{\norm{u}}}= e_1 $, the first standard basis vector and $\inner{\frac{u}{\norm{u}},v} = \inner{\Tilde{\frac{u}{\norm{u}}},\Tilde{v}}$. Note, $v$ and $\Tilde{v}$ have the same distribution. We can see that $\inner{\Tilde{\frac{u}{\norm{u}}},\Tilde{v}} \sim \mathcal{N}(0,\sigma^2)$ since it is the first coordinate of $\Tilde{v}$. By a standard tail bound, we get that for $t > 0$:

\begin{align*}
    \Pr \left[|\inner{\frac{u}{\norm{u}},v}| \geq t  \right]   
    = \Pr \left[| \inner{\Tilde{\frac{u}{\norm{u}}},\Tilde{v}}| \geq t  \right] 
    = \Pr \left[|\Tilde{v}_1| \geq t  \right]  \leq 2 \exp\left(-\frac{t^2}{2\sigma^2}\right)~.
\end{align*}
Therefore
\[\Pr \left[| \inner{u,v}| \geq \norm{u} t  \right]  \leq 2 \exp\left(-\frac{t^2}{2\sigma^2}\right)~.\]
\end{proof}

%2randominnermul
\begin{lemma}
\label{2randominnermul}
    Let $u \sim \mathcal{N}(\zero,\sigma_1^2 I_n)$, and $v \sim \mathcal{N}(\zero,\sigma_2^2 I_n)$. 
    %  $\mathcal{N}(0,\frac{1}{\sqrt{c}}I_n)$. 
    Then, for every $t>0$ we have
        \[ \Pr \left[|\inner{u,v}| \geq \sigma_1 \sqrt{2n} t  \right] \leq e^{-n/16} + 2e^{-t^2/2\sigma_2^2}~.\]
\end{lemma}
\begin{proof}
    
Using Lemma~\ref{lem:randomnormsmall} we get that w.p. $\leq e^{-n/16}$ we have $\norm{u} \geq \sigma_1 \sqrt{2n}$. 
Moreover, by Lemma~\ref{lem:randominnermul}, w.p. $\leq 2 \exp \left(-\frac{t^2}{2\sigma_2^2}\right)$ we have $| \inner{u,v}| \geq \norm{u} t$. 
By the union bound, we get
\begin{align*}
    \Pr \left[|\inner{u,v}| \geq \sigma_1 \sqrt{2n} t  \right] 
    \leq \Pr\left[ \norm{u} \geq \sigma_1 \sqrt{2n} \right] + \Pr \left[ | \inner{u,v}| \geq \norm{u} t \right]
    \leq e^{-n/16} + 2 \exp \left(-\frac{t^2}{2\sigma_2^2} \right)~.
\end{align*}
% Conditioned on this, we get from Lemma \ref{lem:randominnermul} that: 
% \[
% \Pr \left[|\inner{u,v}| \geq \sigma_1 \sqrt{2n} t  \right] \leq 2e^{-t^2/2\sigma_2^2}~.
%\]
% Applying union bound, 
% \[
% \Pr \left[|\inner{u,v}| \geq \sigma_1 \sqrt{2n} t  \right] \leq  e^{-n/16} + 2e^{-t^2/2\sigma_2^2}~. 
% \]
\end{proof}

We are now ready to prove the main theorem of this section:

\begin{proof}[\thmref{thm: exists pert}]

By \thmref{thm:rotation invariant}(2), we can assume w.l.o.g. that $P=M=\text{span}\{e_1,\dots,e_{d-\ell}\}$.
We also assume w.l.o.g. that $y_0=1$, the case $y_0=-1$ is proved in a similar manner. Denote by $\bar{w}:=(w)_{d-\ell+1:d}$, the last $\ell$ coordinates of $w$. By \thmref{thm:weights dont change} we have that $\bar{w}_i$ have not  changed after training from their initial value.
% Hence, also $\bar{w}_i$ fo

We can write $N(x_0 + z)$ as:

% As in previous section, denote by $M$ the $d-\ell$ dimensional linear manifold spanned by $\{e_1,...,e_{d-\ell}\}$. Let $R$ be an orthogonal transformation matrix $R$ given by \lemref{rotatedata} from the manifold $P$ to $M$. We denote by $w^P$ the weights of the network when trained on data $X \subseteq P$, and by $w^M$ the weights trained on $X_R = \{Rx | x \in X\} \subseteq M$. As in the previous section, we get that $w^M$ and $w^P$ have the same distribution up to an orthogonal translation $R$, as well as their respective projections on the orthogonal subspaces $P^\perp$ and $M^\perp$. Since $R$ is an orthogonal translation, it preserved norms and angles. Since the following proof only involves these properties of the weight vectors, we get to the same conclusions for $w^P$ and $w^M$. Thus, for simplicity, we continue the proof for $w:= w^M$ and $\hat{w}_i = \Pi_{M^\perp}\left(w_i\right)$ the projection of $w_i$ on $M^\perp$.

% We assume w.l.o.g. that $y_0 = 1$. The proof is similar for $y=-1$. We look at $N(x_0+z)$ as a sum of two expressions:

\begin{align}
    N(x_0+z) & = \sum_{i=1}^m u_i \sigma(\inner{w_i,x_0} + \inner{w_i,z}) \nonumber\\
    & = \sum_{i  \in I_-} u_i \sigma(\inner{w_i,x_0} + \inner{w_i,z}) + \sum_{i \in I_+} u_i \sigma(\inner{w_i,x_0} + \inner{w_i,z}) \nonumber\\
    & = \sum_{i  \in I_-} u_i \sigma(\inner{w_i,x_0} + \inner{\bar{w}_i,\bar{z}}) + \sum_{i \in I_+} u_i \sigma(\inner{w_i,x_0} + \inner{\bar{w}_i,\bar{z}}) \label{eq:N of x_0 + z} 
\end{align}
where the last equality is since $(z)_{1:d-\ell} = \bm{0}$, hence $\inner{w,z} = \inner{\bar{w},\bar{z}}$ for every $w\in\reals^d$. We will bound each term of the above separately.

For the first term in \eqref{eq:N of x_0 + z}, where $i\in I_-$ we can write:
\begin{align*}
    \inner{\bar{w}_i,\bar{z}} = \alpha\norm{\bar{w}_i}^2  + \alpha\inner{\bar{w}_i, \sum_{j\neq i}\text{sign}(u_j)\bar{w}_j}~.
\end{align*}
By our assumptions, $\bar{w}_i\sim\Ncal\left(\zero, \frac{1}{d}I_\ell\right)$ and $\sum_{j\neq i}\text{sign}(u_j)\bar{w}_j\sim \Ncal\left(\zero, \frac{m-1}{d}I_\ell\right)$, since it is a sum of $m-1$ i.i.d. Gaussian random variables, which are also symmetric hence multiplying them by $-1$ does not change their distribution. From \lemref{randomnormlarge} we get w.p. $\geq 1- e^{-\ell/16}$ that 
\[
\alpha\cdot  \norm{\bar{w}_i}^2 \geq \alpha\cdot \frac{\ell}{2d}~.
\]
From \lemref{2randominnermul}, and using $t = \sqrt{\frac{(m-1)\log(dm^2)}{d}}$ we get w.p. $ \geq 1 -e^{-\ell/16} + 2e^{-t^2d/2(m-1)} = 1 -e^{-\ell/16} + 2m^{-1}d^{-1/2}$ that
% w.p. $ \geq 1 -e^{-\ell/16} + 2m^{-1}d^{-1/2}$ that 
\begin{align}\label{eq:inner product sum w_i bound}
\inner{\bar{w}_i, \sum_{j\neq i}\text{sign}(u_j)\bar{w}_j} &\leq \frac{1}{\sqrt{d}}t\sqrt{2\ell} \nonumber \\
& = \frac{1}{d}\cdot \sqrt{2\ell(m-1)\log(m^2d)}~.
\end{align}
% \begin{equation}\label{eq:inner product sum w_i bound}
% \inner{\bar{w}_i, \sum_{j\neq i}\text{sign}(u_j)\bar{w}_j} \leq \frac{1}{d}\cdot \sqrt{2\ell(m-1)\log(m^2d)}~.    
% \end{equation}
Applying union bound over the above two events, and for every $i\in I_-$, we get w.p. $\geq 1 - 2\left( me^{-\ell/16} + d^{-1/2}\right)$ that:
\[
\inner{\bar{w}_i,\bar{z}} \geq \frac{\alpha \ell}{2d} - \frac{\alpha}{d}\sqrt{2\ell(m-1)\log(m^2d})~.
\]

For the second term in \eqref{eq:N of x_0 + z}, where $i\in I_+$ we can write in a similar way:
\[
    \inner{\bar{w}_i,\bar{z}} =  - \alpha\norm{\bar{w}_i}^2  + \alpha\inner{\bar{w}_i, \sum_{j\neq i}\text{sign}(u_j)\bar{w}_j}~.
\]
Using the same argument as above, we get w.p $\geq 1 - 2\left( me^{-\ell/16} + d^{-1/2}\right)$ that:
\[
\inner{\bar{w}_i,\bar{z}} \leq -\frac{\alpha \ell}{2d} + \frac{\alpha}{d}\sqrt{2\ell(m-1)\log(m^2d)}~.
\]
By assuming that $\ell \geq 8(m-1)\log(m^2d)$ we get that $\inner{\bar{w}_i,z} \leq 0$. Denote $C:= \frac{\alpha \ell}{2d} - \frac{\alpha}{d}\sqrt{2\ell(m-1)\log(m^2d)}$, then going back to \eqref{eq:N of x_0 + z}, using the above bounds and applying union bound, we get w.p. $\geq 1 - 4\left( me^{-\ell/16} + d^{-1/2}\right)$ that:
\begin{align*}
    N(x_0 + z) & \leq \sum_{i\in I_-}u_i \sigma(\inner{w_i,x_0} + C) + \sum_{i\in I_+}u_i \sigma(\inner{w_i,x_0})  \\
    & = \sum_{i\in I_-}u_i \sigma(\inner{w_i,x_0} + C) + \sum_{i\in I_+}u_i \sigma(\inner{w_i,x_0}) + \sum_{i\in I_-}u_i\sigma(\inner{w_i,x_0}) - \sum_{i\in I_-}u_i\sigma(\inner{w_i,x_0}) \\
    & =\sum_{i\in I_-}u_i \sigma(\inner{w_i,x_0} + C) - \sum_{i\in I_-}u_i\sigma(\inner{w_i,x_0}) + N(x_0) \\
    & = \sum_{i\in I_-}u_i\left(\sigma(\inner{w_i,x_0} + C) - \sigma(\inner{w_i,x_0})\right) + N(x_0)~.
\end{align*}
Define $F_-:=\{i\in I_-: \inner{w_i,x_0} \geq 0\}$, and $k_- = |F_-|$. We have that:

\begin{align*}
\sum_{i\in I_-}u_i\left(\sigma(\inner{w_i,x_0} + C) - \sigma(\inner{w_i,x_0})\right) & \leq \sum_{i\in F_-}u_i\left(\sigma(\inner{w_i,x_0} + C) - \sigma(\inner{w_i,x_0})\right) \\
& = \sum_{i\in F_-}u_i C = -\frac{k_- C}{\sqrt{m}}~,
\end{align*}
where the first inequality is since we only sum over negative terms, and the second inequality is since both $\inner{w_i,x_0} \geq 0$ (because $i\in F_-$) and $C\geq 0$ (because $\ell \geq 32(m-1)\log(m^2d)$). Combining all of the above, we get that:

\begin{equation}\label{eq: N of x_0 + z w.r.t N of x_0}
N(x_0 + z) \leq -\frac{k_- C}{\sqrt{m}} + N(x_0)~.    
\end{equation}

By our assumption that $\ell \geq 32(m-1)\log(m^2 d)$ we have that 
\begin{align*}
    C &= \alpha \left( \frac{1}{2} \frac{\ell}{d} - \sqrt{2} \sqrt{m-1} \frac{\sqrt{\ell}}{d} \sqrt{\log(dm^2 )} \right) \\
    & = \frac{\alpha\sqrt{\ell}}{d}\left(\frac{\sqrt{\ell}}{2} - \sqrt{2(m-1)\log(m^2 d)}\right) \\
    & \geq  \frac{\alpha\ell}{4d}~.       
\end{align*}
Plugging in $C$ and $\alpha = \frac{8\sqrt{m}dN(x_0)}{k_- \ell}$ to \eqref{eq: N of x_0 + z w.r.t N of x_0} we get that:
\begin{align*}
N(x_0 + z) &\leq -\frac{k_- C}{\sqrt{m}} + N(x_0) \\
& \leq -\frac{k_-}{\sqrt{m}}\cdot \frac{\ell}{4d}\cdot \frac{8\sqrt{m}dN(x_0)}{k_- \ell} + N(x_0) = -N(x_0) < 0~,    
\end{align*}
and in particular $\text{sign}(N(x_0)) \neq \text{sign}(N(x_0 + z))$.

We are left with calculating the norm of $z$:
\begin{align*}
\norm{z} &= \alpha\cdot\left\| \sum_{i\in I_-} \Pi_{M^\perp}(w_i) - \sum_{i\in I_+} \Pi_{M^\perp}(w_i) \right\| \\
&=\alpha\cdot \left\|\sum_{i=1}^m -\text{sign}(u_i)\Pi_{M^\perp}(w_i)\right\| \\
&= \alpha\cdot \left\|\sum_{i=1}^m -\text{sign}(u_i)\bar{w}_i\right\|~.
&
\end{align*}

Since for each $i\in[m]$, $\bar{w}_i\sim\Ncal\left(\zero,\frac{1}{d}I_\ell\right)$, then $-\text{sign}(u_i)\bar{w}_i$ also have the same distribution, because this is a symmetric distribution. Hence, $\sum_{i=1}^m -\text{sign}(u_i)\bar{w}_i\sim \Ncal\left(\zero,\frac{m}{d}I_\ell\right)$ as a sum of Gaussian random variables. Using \lemref{lem:randomnormsmall} we get w.p $\geq 1-e^{-\ell/16}$ that $\left\|\sum_{i=1}^m -\text{sign}(u_i)\bar{w}_i\right\|^2 \leq \frac{2 m\ell}{d}$. Plugging in $\alpha$ we get that:
\begin{align*}
    \norm{z} \leq \sqrt{\frac{2m\ell}{d}}\cdot \frac{8\sqrt{m}dN(x_0)}{k_- \ell} =
    8\sqrt{2}N(x_0)\cdot \frac{m}{k_-}\cdot \sqrt{\frac{d}{\ell}}~.
    % \frac{8\sqrt{2}m\sqrt{\ell}N(x_0)}{k_-\sqrt{d}}
\end{align*}

\end{proof}

\section{Proofs for \secref{sec: improve robustness}}\label{apen: proofs for improve robustness}

For proving the main theorem, we will use 
%\lemref{lem:randomnormsmall} and 
the following lemma that upper bounds the norm of a sum of Gaussian random variables:

\begin{lemma}
\label{lem:sumofrandomlydirectedupper}
    Let $w_1,..,w_m \in\reals^n$ such that for all $i \in [m]$, $w_i \sim \mathcal{N}(\zero,\sigma^2 I_n)$, then we have: %(for $\sigma = \frac{1}{\sqrt{c}})$ i.e. m n-dimensional vectors randomly directed in the n-dimensional sphere. Then - 
    % \begin{enumerate}
    %     \item $\mathbb{E} \left[ \norm{\sum\limits_{i=1}^m w_i}^2 \right] = m \cdot \mathbb{E} \left[  \norm{w_i}^2\right]  = m \sigma^2 n (= \frac{mn}{c} )$
     \[
         \mathbb{P} \left[\norm{\sum\limits_{i=1}^m w_i}^2 \geq 2 m \sigma^2 n \right] \leq e^{-\frac{n}{16}} 
     \]
    % \end{enumerate}
\end{lemma}

\begin{proof}
    We denote the $j$-th coordinate of the vector $w_i \in\reals^n$ by $w_{i,j}$. Note, for any $i \in [m]$ and $j \in [n]$ we have $w_{i,j} \sim \mathcal{N}(0,\sigma^2)$. We denote by $s$ the sum vector $s :=\sum\limits_{i=1}^m w_{i}$, and by $s_j$ the $j$-th coordinate of $s$. By this definition, $s_j  = \sum\limits_{i=1}^m w_{i,j}$ is a sum of $m$ independent Gaussian variables and therefore also a Gaussian variable. Therefore, $s \sim \mathcal{N}(\zero,m \sigma^2 I_n)$.  We use Lemma \ref{lem:randomnormsmall} with variance $m \sigma^2$ and get that:
    % $$\mathbb{E} \left[ \norm{\sum\limits_{i=1}^m w_i}^2 \right] =\mathbb{E} \left[ \norm{s}^2 \right] =  m \sigma^2 n =  m \cdot \mathbb{E} \left[  \norm{w_i}^2\right] $$
    % And - 
    \[\mathbb{P} \left[\norm{\sum\limits_{i=1}^m w_i}^2 \geq 2 m \sigma^2 n \right] \leq e^{-\frac{n}{16}}~.\]

\end{proof}

We now prove the main theorem of this section:

\begin{proof}[Proof of \thmref{thm:upper bound small init}]
Similar to the lower bound of the norm, let $M = \text{span}\{e_1,\dots,e_{d-\ell}\}$. By \thmref{thm:rotation invariant}(1), given a training dataset $X\subseteq P$, it is enough to consider a training set $X_R = \{Rx: x\in X\}$, where $R$ is an orthogonal matrix such that $R\cdot P = M$, and training is done over $X_R$. From now on, we assume that the training data, as well as $x_0$ lie on $M$, and the consequences of this proof would also imply for a dataset $X$ and $x_0\in P$.

The projection of the gradient on $M^\perp$ is equal to:

 \[
    \Pi_{M^\perp}\left(\frac{\partial N(x_0)}{\partial x}\right) = \Pi_{M^\perp}\left(\sum\limits_{i=1}^{m} u_i w_i \mathbbm{1}_{\inner{w_i,x_0} \geq 0}\right) = \sum\limits_{i=1}^{m} \Pi_{M^\perp}\left(u_i w_i\right)  \mathbbm{1}_{i \in S} = \sum\limits_{i \in S} \Pi_{M^\perp}\left(u_i w_i\right)~.
\]

%  And for 
%  \[
%     \hat{w}_{i,j} = 
% \begin{cases}
%     0& \text{if } j\leq d-\ell\\
%     w_{i,j}              & \text{otherwise}
% \end{cases}
% \]\\

% we can write:
% \[\Pi_{M^\perp}\left(\frac{\partial N(x_0)}{\partial x_0}\right) = \sum\limits_{i \in S} \hat{w}_i u_i  \]

Denote by $\hat{w}_i = (w_i)_{d-\ell+1:d}$, the last $\ell$ coordinates of $w_i$. By \thmref{thm:weights dont change} we get that for every $i\in[m]$, $\hat{w}_i$ did not change from their initial value during training.

Recall that we initialized $\hat{w}_i \sim \mathcal{N}(\zero,\beta^2 I_{\ell})$.
Note that the set $S$ is independent of the value of the $\hat{w}_i$'s. This is because $\hat{w}_i$ does not effect the training, hence will not effect $w_i - \Pi_{M^\perp}(w_i)%szA
$. Also, after choosing $x_0$ we have $\inner{\hat{w}_i,\hat{x}_0} = 0$, since $\hat{x}_0 = \bm{0}$, which means that the choice of $S$ is independent of the $\hat{w}_i$'s. We can conclude that the random variables $\hat{w}_i$ for $i\in S$ are sampled independently.

Therefore, from Lemma~\ref{randomlydirected} we get that w.p. $\geq 1- e^{-\ell/16}$: 

\[ 
    \norm{  \sum\limits_{i \in S} \hat{w}_i} \leq \beta \sqrt{2 k \ell}~.
\]
Note, since for all $i \in [m]$, $|u_i| = \frac{1}{\sqrt{m}}$ and they are not trained, we get w.p. $\geq 1-e^{-\ell/16}$ that:
\[
    \left\|\Pi_{M^\perp}\left(\frac{\partial N(x_0)}{\partial x}\right) \right\| \leq \beta \sqrt{\frac{2k \ell}{m}}~.
\]
\end{proof}

\subsection{Explicit $L_2$ regularization}\label{apen: explicit regularization}
\begin{proof}[Proof of \thmref{thm: explicit regularization gradient}]
 As before, for this proof we rotate the data subspace $P$ to lie on $M = \text{span}\{e_1,\dots,e_{d-\ell}\}$ and rotate the model's weights accordingly. For a dataset $(x_1,y_1),..,(x_r,y_r)$, we train over the following objective:

 \[ \sum_{j=1}^r L(y_j \cdot N(x_j,\bw_{1:m}))) + \frac{1}{2}\lambda \norm{\bw_{1:m}}^2\]
 
 % For a rotated data point $x \in X$ labeled as $y$, we define the loss function w.r.t. the weights $\bw_{1:m}$: $L_R(N(x,\bw_{1:m})\cdot y, \bw_{1:m}) = L(N(x,\bw_{1:m})\cdot y)  + \lambda \norm{\bw_{1:m}}^2$, where $L(N(x,\bw_{1:m})\cdot y) = \log(1+e^{-N(x,\bw_{1:m})\cdot y})$

In \thmref{thm:weights dont change}, we showed for all $(x_j, y_j)$ that if we train the model using the loss $L$ we get:

\[\Pi_{M^\perp}\left(\frac{\partial L\left( N(x_j,\bw_{1:m})\cdot y_j\right)}{\partial w_i}\right) = 0 \]

Now, we analyze the training process using the new loss which includes the regularization term. We denote by $w^{(t)}_i$ the weight vector $w_i$ after $t$ training steps, and by $\hat{w}^{(t)}_i := \Pi_{M^\perp}\left(w^{(t)}_i\right) $ its projection on the subspace orthogonal to $M$.
We look at the projected gradient of $w^{(t)}_i$ w.r.t. the loss:

\begin{align*}
    % \Pi_{M^\perp}\left(\frac{\partial L_R(N(x,\bw^t_{1:m})\cdot y, w_{1:m})}{\partial w^t_i}\right) =& 
    \Pi_{M^\perp}&\left(\frac{\partial \sum_{j=1}^r L\left( N(x_j,\bw^{(t)}_{1:m})\cdot y_j\right)}{\partial w_i} + \frac{\partial \frac{1}{2}\lambda \norm{w^{(t)}_i}^2}{\partial w_i}\right)=\\
    =&\sum_{j=1}^r\Pi_{M^\perp}\left(\frac{\partial L\left( N(x_j,\bw^{(t)}_{1:m})\cdot y_j\right)}{\partial w_i}\right) + \Pi_{M^\perp}\left(\frac{\partial \frac{1}{2}\lambda \norm{w^{(t)}_i}^2}{\partial w_i}\right)\\
    =& \Pi_{M^\perp}\left(\frac{\partial \frac{1}{2}\lambda \norm{w^{(t)}_i}^2}{\partial w_i}\right)\\
    =& \Pi_{M^\perp}\left(\lambda w^{(t)}_i\right) \\
    =& \lambda \hat{w}^{(t)}_i.
\end{align*}
For a training step of size $\eta$, using gradient descent we get that:

\[ \hat{w}^{(t+1)}_i = \hat{w}^{(t)}_i - \eta \lambda \hat{w}^{(t)}_i.  \]
Thus, after a total of $T$ iteration of training we get that:

\[ \hat{w}^{(T)}_i = (1- \eta \lambda)^T \hat{w}^{(0)}_i~.  \]

Therefore, the projection of gradients after training onto $P^\perp$ will be the same as if they were initialized to $\sim \mathcal{N}\left(0, \frac{(1- \eta \lambda)^{2T}}{d}I_d\right)$ and trained using logistic loss without regularization. The rest of the proof is the same as \thmref{thm:upper bound small init} for $\beta = \frac{(1- \eta \lambda)^T}{\sqrt{d}}$.
\end{proof}
% \section{Details for \secref{sec: experiments}}\label{appen: experiments}
\section{Further Experiments and Experimental Details}\label{appen: experiments}

\subsection{Further Experiments}

In Figure \ref{fig:2d data def init}  we present the boundary of a two-layer ReLU network trained over a $25$-point dataset on a two-dimensional linear subspace, similar to \figref{fig:2d data}. We train the networks until reaching a constant positive margin. The difference between the figures is that in \figref{fig:2d data def init} we initialize the weights using the default PyTorch initialization, while in \figref{fig:2d data} we initialized using a smaller scale for the robustness effect to be smaller, and visualized more easily.  
The experiment in \figref{fig:2d data def init} is demonstrating an extreme robustness effect, occurring when using the standard settings.

\begin{figure}[!ht]
\centering
\begin{subfigure}[b]{0.3\textwidth}
    \centering
     \includegraphics[width=\textwidth]{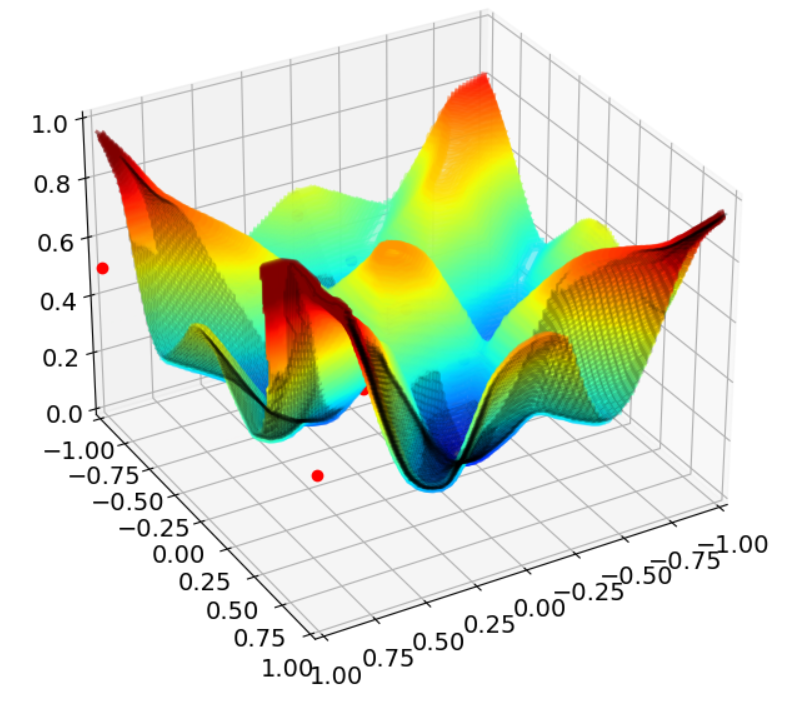}
     \caption{}
     \label{fig: 2d clean def init}
\end{subfigure}
\begin{subfigure}[b]{0.3\textwidth}
    \centering
     \includegraphics[width=\textwidth]{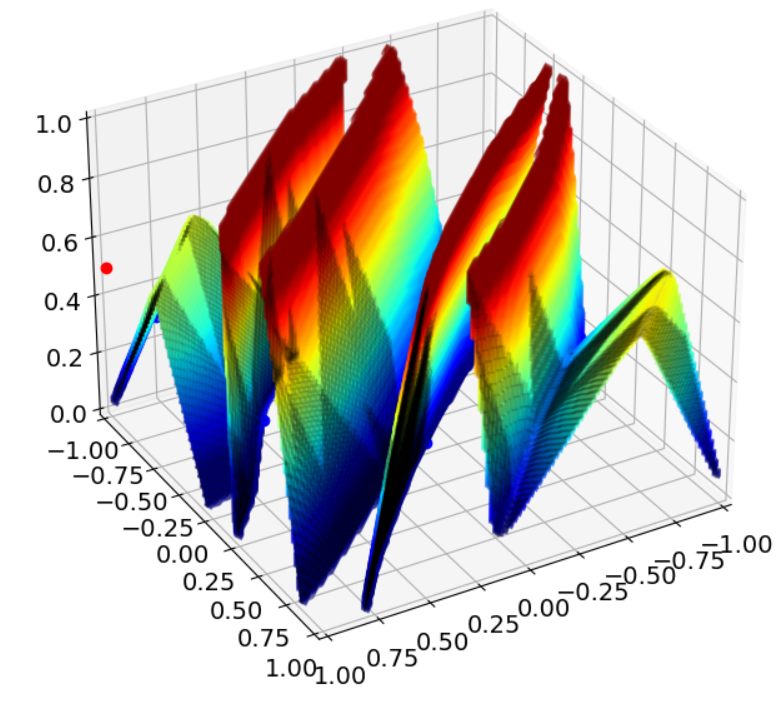}
     \caption{}
     \label{fig: 2d small init def init}
\end{subfigure}
\begin{subfigure}[b]{0.3\textwidth}
    \centering
     \includegraphics[width=\textwidth]{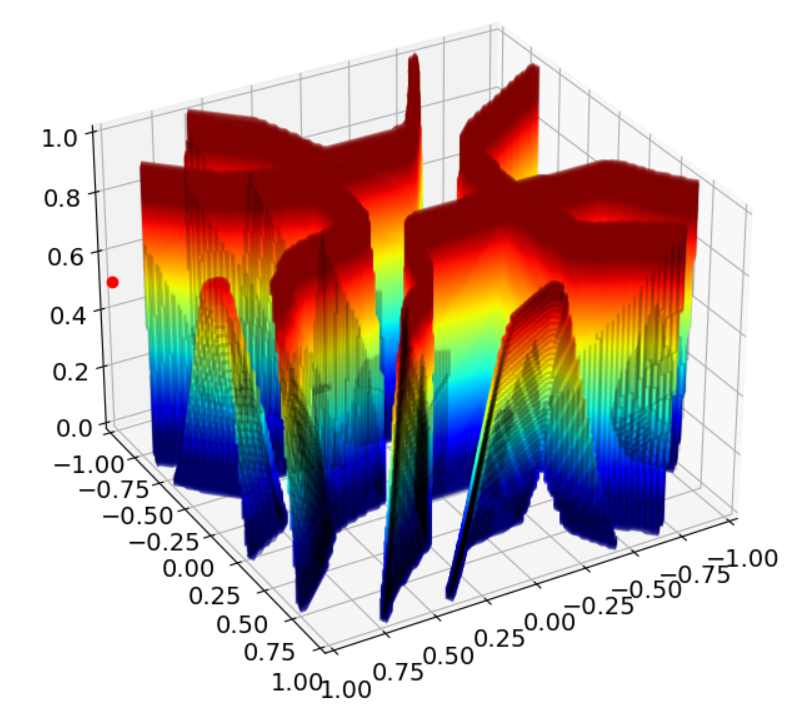}
     \caption{}
     \label{fig: 2d regular def init}
\end{subfigure}
% \subfigure[]{\includegraphics[width=0.3\textwidth]{imgs/subfigures/defultinit/chess-clean.png}\label{fig: 2d clean def init}}
% \subfigure[] {\includegraphics[width=0.3\textwidth]{imgs/subfigures/defultinit/chess-smallinit.png}\label{fig: 2d small init def init}}
% \subfigure[]{\includegraphics[width=0.3\textwidth]{imgs/subfigures/defultinit/chess-ragular.png}\label{fig: 2d regular def init}}
\caption{\textbf{Experiments on two-dimensional dataset.} We plot the dataset points and the decision boundary in 3 settings: (a) Vanilla trained network, (b) The network's weights are initialized from a smaller variance distribution, and (c) Training with regularization. Colors are used to emphasise the values in the $z$ axis.} \label{fig:2d data def init}
\end{figure}

In \figref{fig:1d data deep} we go beyond the theory discussed in this paper, and present similar phenomena in all three settings for a five-layer ReLU network. In \figref{fig: 1d deep clean} we can see the boundary of the regularly trained network within a small distance in $P^\perp$ from the data points. In \figref{fig: 1d deep small init} we use small initialization for all five layers, and present a boundary almost orthogonal to the data manifold. In \figref{fig: 1d deep regular}, the boundary of a regularized trained network is in a similar form. 
This experiment suggests that our theoretical results might be extended also to deeper networks, where all layers are trained.
%Full details for all experiments can be found in \appref{appen: experiments}.

% \begin{figure}[htbp]
\begin{figure}[!ht]
\centering
\begin{subfigure}[b]{0.3\textwidth}
     \centering
     \includegraphics[width=\textwidth]{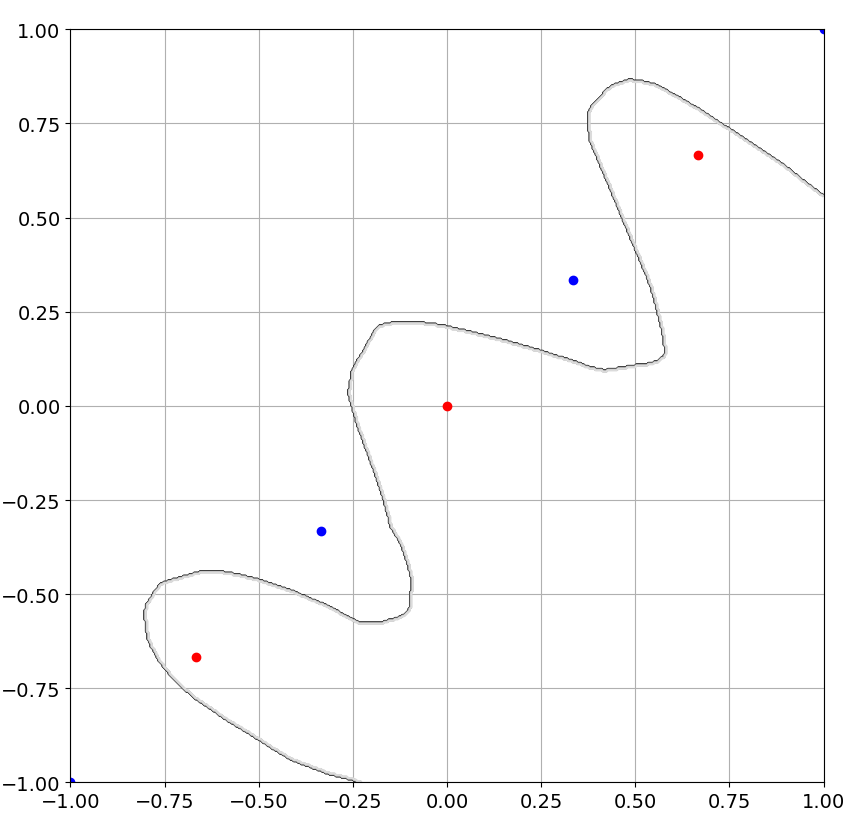}
     \caption{}
     \label{fig: 1d deep clean}
\end{subfigure}
\begin{subfigure}[b]{0.3\textwidth}
     \centering
     \includegraphics[width=\textwidth]{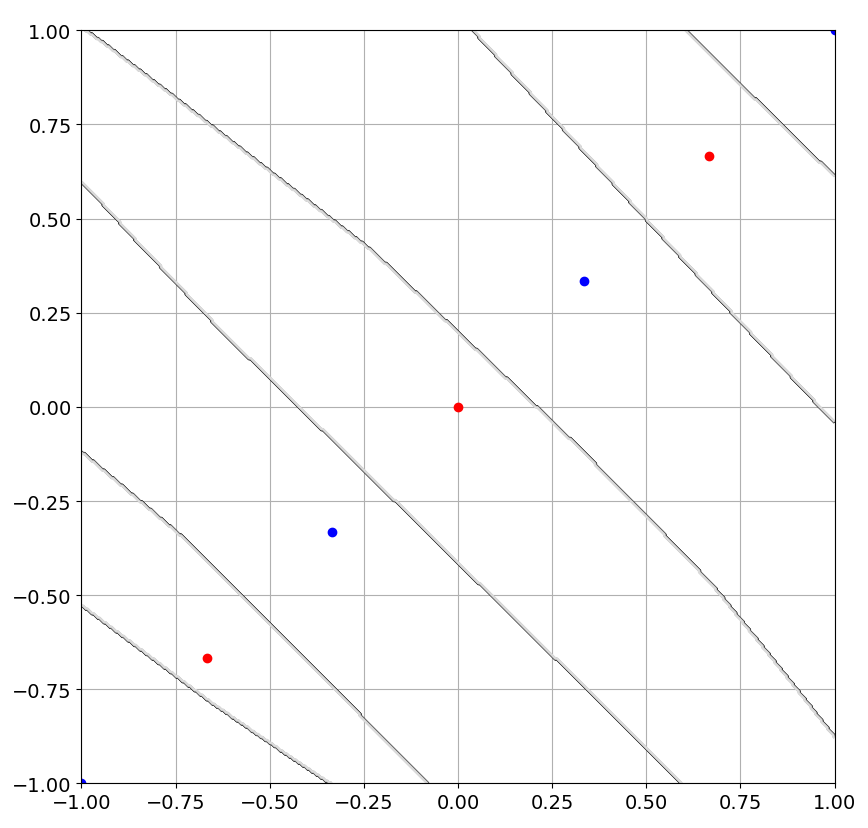}
     \caption{}
     \label{fig: 1d deep small init}
\end{subfigure}
\begin{subfigure}[b]{0.3\textwidth}
     \centering
     \includegraphics[width=\textwidth]{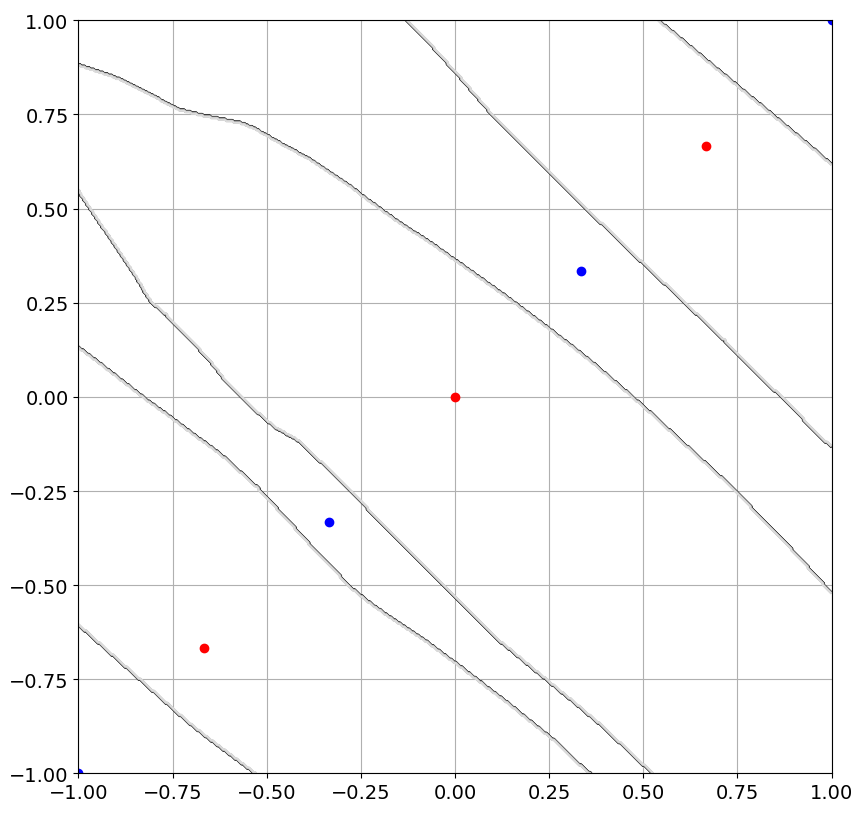}
     \caption{}
     \label{fig: 1d deep regular}
\end{subfigure}

% \subfigure[]{\includegraphics[width=0.3\textwidth]{imgs/subfigures/defultinit/5layer-clean.png}\label{fig: 1d deep clean}}
% \subfigure[] {\includegraphics[width=0.3\textwidth]{imgs/subfigures/defultinit/5layer-smallinit.png}\label{fig: 1d deep small init}}
% \subfigure[]{\includegraphics[width=0.3\textwidth]{imgs/subfigures/defultinit/5layer-regulr.png}\label{fig: 1d deep regular}}
\caption{\textbf{Experiments on one-dimensional dataset with deep network.} We plot the dataset points and the decision boundary in 3 settings: (a) Vanilla trained network, (b) The network's weights are initialized from a smaller variance distribution, and (c) Training with regularization.} \label{fig:1d data deep}
\end{figure}

\subsection{One-dimensional dataset experiment - 2 layer network (\figref{fig:1d data})}
\paragraph{Dataset} For all the three experiments we used a 7-point data set, spread equally on the two dimensional line $y=x$ from $(-1,-1)$ to $(1,1)$.
\paragraph{Network} For all the three experiments we used two-layer ReLU network of width $100$ with biases in both layers. The weights of both layers were initialized using (1+3) default PyTorch initialization for linear layers, (2) default initialization divided by $3$.
\paragraph{Training} We used train step of size $0.02$ for (1+3) and $0.04$ for (2). We trained both layers until the margin reached $0.3$. The losses we used were (1+2) Logistic loss, (3) Logistic loss with $0.005$ $L_2$ regularization.

\subsection{Two-dimensional dataset experiment - smaller effect (\figref{fig:2d data})}
\paragraph{Dataset} For all the three experiments we used a 25-point data set, spread equally on a grid which lies on the $z=0.5$ axis.
% a Cartesian multiplication of 5 points spread equally in [0,1].
\paragraph{Network} For all the three experiments we used two-layer ReLU network of width $4000$ with biases in both layers. The weights in the first layer were initialized in (1+3) from $\mathcal{N}(\zero,1/3 I_3)$, and in (2) from $\mathcal{N}(\zero,1/36 I_3)$. The weight of the output layer were initialized to the uniform distribution over the set $\{-1,1\}$.
\paragraph{Training} For all the experiments we trained both layers until the margin reached $0.3$ and we used train step of size $0.002$. The losses we used were (1+2) Logistic loss, (3) Logistic loss with $0.8$ $L_2$ regularization on the weights of the first layer.

\subsection{Two-dimensional dataset experiment (\figref{fig:2d data def init})}
\paragraph{Dataset} For all the three experiments we used a 25-point data set, spread equally on a grid which lies on the $x - y$ axis.
% a Cartesian multiplication of 5 points spread equally in [0,1].
\paragraph{Network} For all the three experiments we used two-layer RelU network of width $400$ with biases in both layers. The weights in any layer were initialized using (1+3) default PyTorch initialization for linear layers, (2) default initialization divided by $3$.
\paragraph{Training} For (1) experiments we used train step of size $0.005$, and for (2+3) we used step of size $0.05$. We trained both layers until the margin reached $0.1$. The losses we used were (1+2) Logistic loss, (3) Logistic loss with $0.005$ $L_2$ regularization.

\subsection{One-dimensional dataset experiment - 5 layer network (\figref{fig:1d data deep})}
\paragraph{Dataset} For all the three experiments we used a 7-point data set, spread equally on the two dimensional line $y=x$ from $(-1,-1)$ to $(1,1)$.
\paragraph{Network} For all the three experiments we used 5-layer RelU network of width $100$ with biases in all layers. The weights in any layer were initialized using (1+3) default PyTorch initialization for linear layers, (2) default initialization divided by $3$.
\paragraph{Training} For (1+3) experiments we used train step of size $0.02$, and for (2) we used step of size $0.06$. we trained all layers until the margin reached $0.3$. The losses we used were (1+2) Logistic loss, (3) Logistic loss with $0.01$ $L_2$ regularization.

% 2 layers: 
% 2deep clean & regul & small init - 
% 	step size 0.05
% 	margin 0.3 (no sigmoid)
% 	training both layers
% 	with bias in both
% 	width 100

%  deep:
%  step size 0.005
% 	margin 0.3 (no sigmoid)
% 	training all layers
% 	with bias in all
% 	width 100

%  chess 2 layers:
%         margin 0.3
% 		step size 0.002
% 		width 4000 

\end{document}